\documentclass[12pt]{article}
\usepackage[utf8]{inputenc}
\usepackage{comment, natbib}

\usepackage{xcolor, hyperref}
\hypersetup{
    colorlinks,
    linkcolor={blue!50!black},
    citecolor={blue!50!black},
    urlcolor={blue!80!black}
}
\usepackage{chngcntr}
\usepackage{mathtools, amsthm, amssymb, bbm} 

\usepackage{tikz}
\usetikzlibrary{shapes.geometric, arrows}

\usepackage{xcolor}

\newcommand{\blind}{0}

\renewcommand{\t}{t}
\newcommand{\GPpredictCDF}{\Phi_{\mu, \sigma^2}}

\newcommand{\GPpredictmean}{\mu}
\newcommand{\GPpredictsd}{\sigma}
\newcommand{\GPpredictvar}{\sigma^2}
\newcommand{\GPobs}{y}

\newcommand{\GPupdate}{\boldsymbol{\xi}_{n+1}} 
\newcommand{\GPupdateU}{\xi_{n+1}} 
\newcommand{\GPupdateM}{m_{n+1}}
\newcommand{\GPupdateK}{k_{n+1}}
\newcommand{\GPupdateO}{\mathcal{A}_{n+1}}

\newcommand{\evaI}{\overline{\mathrm{CRPS}}_{\gamma_1}}
\newcommand{\evaP}{\overline{\mathrm{CRPS}}_{\gamma_2}}
\newcommand{\evaB}{\overline{\mathrm{CRPS}}}

\newcommand{\seqIp}{\mathrm{CRPS}_{\gamma_1}}
\newcommand{\seqIs}{\mathrm{ICRPS}_{\gamma_1}}

\newcommand{\seqPp}{\mathrm{CRPS}_{\gamma_2}}
\newcommand{\seqPs}{\mathrm{ICRPS}_{\gamma_2}}

\newcommand{\twcrps}{\mathrm{CRPS}_{\gamma}}

\newcommand{\wI}{\gamma_1}
\newcommand{\wP}{\gamma_2}

\newcommand{\sigmaY}{\sigma_{\gamma}}

\newcommand{\intB}{\mathrm{d}\gamma(u)}
\newcommand{\intI}{\mathrm{d}\gamma_1(u)}
\newcommand{\intP}{\mathrm{d}\gamma_2(u)}

\allowdisplaybreaks

\usepackage{dsfont}             

\allowdisplaybreaks

\newtheorem{theorem}{Theorem}

\begin{document}

\def\spacingset#1{\renewcommand{\baselinestretch}%
{#1}\small\normalsize} \spacingset{1}

\if0\blind
{\title{\bf CRPS-Based Targeted Sequential Design with Application in Chemical Space}
\author{Lea Friedli\thanks{Corresponding author (lea.friedli@hotmail.com).
  Athénaïs Gautier and David Ginsbourger acknowledge support from the Swiss National Science Foundation project number 178858. }\hspace{.2cm} \\
    \small{Institute of Mathematical Statistics and Actuarial Science, University of Bern, Switzerland} \\[5pt]
    Athénaïs Gautier \\
    \small{DTIS, ONERA, Paris-Saclay University, 91120, Palaiseau, France} \\[5pt]
    Anna Broccard \\
    \small{Institute of Mathematical Statistics and Actuarial Science, University of Bern, Switzerland}\footnote{Master student at the University of Bern at the time of her main contributions. } \\[5pt]
    David Ginsbourger \\  \small{Institute of Mathematical Statistics and Actuarial Science, University of Bern, Switzerland}  }
    \date{}

  \maketitle \
} \fi

\begin{abstract}
Sequential design of real and computer experiments via Gaussian Process (GP) models has proven useful for parsimonious, goal-oriented data acquisition purposes. In this work, we focus on acquisition strategies for a GP model that needs to be accurate within a predefined range of the response of interest. Such an approach is useful in various fields including synthetic chemistry, where finding molecules with particular properties is essential for developing useful materials and effective medications. GP modeling and sequential design of experiments have been successfully applied to a plethora of domains, including molecule research. Our main contribution here is to use the threshold-weighted Continuous Ranked Probability Score (CRPS) as a basic building block for acquisition functions employed within sequential design. We study pointwise and integral criteria relying on two different weighting measures and benchmark them against competitors, demonstrating improved performance with respect to considered goals. The resulting acquisition strategies are applicable to a wide range of fields and pave the way to further developing sequential design relying on scoring rules. Supplemental material to augment this article is available online. 
\end{abstract}

\noindent%

\spacingset{2}

\section{Introduction}

Gaussian Process (GP) modeling, originating from the geostatistical interpolation method known as kriging \citep{krige1951statistical}, extends its application beyond geostatistics \citep{rasmussen_williams}. Its fundamental concept revolves around modeling a target function $f: \mathcal{X} \rightarrow \mathbb{R}$ from a set of input locations $\mathcal{X}$ to the real numbers; in the case of co-kriging or multitask GP models, the function can also produce vector-valued outputs. Given a finite number of evaluations of $f$ or of by-products thereof (e.g., evaluations of derivatives of $f$), the objective is typically to predict $f$ at arbitrary locations. GP modeling is favored for its suitability with a limited amount of training data, provision of a full predictive distribution enabling uncertainty quantification, and capability to also handle non-Euclidean inputs through appropriate covariance kernels. In terms of the latter, the selection of kernels can be tailored to specific data types and applications. In scenarios where data collection is resource-intensive, efficiently selecting and expanding the training dataset is crucial. In this context, we employ targeted sequential design strategies, with a focus on an excursion set of the form $\Gamma := \{x \in \mathcal{X}: f(x) \geq \t\}$, using some prescribed thereshold $\t \in \mathbb{R}$. In doing so, we follow the now established framework of GP-based excursion set estimation \citep[e.g.,][and references therein]{French2013SpatiotemporalEL, chevalier2013estimating, Bolin2012ExcursionAC, azzimonti2016contributions} and focus on two interconnected elements: suitable evaluation metrics and targeted acquisition criteria. \\

Scoring rules \citep{gneiting2007strictly} are widely used tools for assessing probabilistic forecasts by comparing a predictive distribution with a single value (the materialized response). Scoring rules evaluate both the statistical consistency between predictions and observations (calibration) and the sharpness of the prediction. A scoring rule is called proper if no alternative forecast distribution can yield a lower expected score than the true data-generating distribution. It is strictly proper if the expected score is uniquely minimized by the true data-generating distribution. Because proper scoring rules promote careful evaluations and are firmly rooted in fundamental decision-theoretic principles, they are regarded as essential tools in scientific practice \citep{gneiting2007strictly, brocker2007scoring, thorarinsdottir2013using}. Furthermore, links between scoring rules and divergences have been actively investigated in the literature of forecast evaluation \citep{thorarinsdottir2013using}. In this study, we focus on the widely used Continuous Ranked Probability Score (CRPS; \citeauthor{matheson1976scoring}~\citeyear{matheson1976scoring}). There are several advantages that motivate using CRPS \citep{gneiting2007probabilistic, waghmare2025proper}: Echoing with the above, the CRPS has been proven to define a proper scoring rule, and to be strictly proper relative to probability measures that have finite first moment. Furthermore, the CRPS is distance sensitive, in the sense that it rewards forecasts that allocate probability mass closer to the realized observation; and it enjoys theoretical uniqueness as it is the only scoring role satisfying translation invariance, homogeneity, and kernel representability.\\

An interesting category of scoring rules for goal-oriented prediction evaluation consists of weighted scoring rules, allowing to emphasize particular outcomes by incorporating a weighting measure \citep{gneiting2011comparing}. For the CRPS, \citet{allen2023evaluating} outline three potential approaches to use a weighting measure. Here, we consider the threshold-weighted CRPS introduced by \citet{matheson1976scoring} and \citet{gneiting2011comparing}, which has been widely used for evaluating forecasts of extreme weather events \citep[e.g.,][]{lerch2013comparison}. To reflect the special importance of predictions in/near the excursion set $\Gamma$, we consider both an indicator weight on $\Gamma$, and a Gaussian weight centered at $\t$.  \\

In terms of data acquisition, we aim to extend the training set by strategically adding promising candidates, thereby improving the model with respect to our goals. While in the field of Bayesian global optimization \citep{Kushner1964ANM, Mockus1978, Mockus1989BayesianAT}, the focus is on finding the optimum or an optimizer of $f$, we are focusing on the excursion set $\Gamma$. We rely on a sequential design strategy with selecting each training point as if it were the final one, known as a one-step look-ahead (or myopic) strategy. There exist different criteria to choose the next point, depending on the objective of interest. For a focus on the excursion set $\Gamma$, we distinguish between pointwise criteria, and integral criteria. Pointwise criteria focus solely on local GP predictions, with popular examples including the Targeted Mean Square Error (TMSE; \citeauthor{picheny2010adaptive}~\citeyear{picheny2010adaptive}) and entropy-based criteria \citep{cole2023entropy}. Integral criteria consider the predictive distribution over the entire domain $\mathcal{X}$ and well-known approaches include the Targeted Integrated Mean Square Error (TIMSE; \citeauthor{picheny2010adaptive}, \citeyear{picheny2010adaptive}) and Stepwise Uncertainty Reduction (SUR; \citeauthor{bect2012sequential}~\citeyear{bect2012sequential}) criteria such as the Integrated Bernoulli Variance (IBV; \citeauthor{bect2019supermartingale}, \citeyear{bect2019supermartingale}). \\

Given that the threshold-weighted CRPS, with an appropriate weighting measure, evaluates the GP model's performance while emphasizing excursions, it also provides a promising foundation for designing targeted acquisition criteria. However, in their original form, scoring rules compare the prediction with the materialized response, which is unavailable during data acquisition. Yet, one may rely on the expected scoring rule by integrating the score of forecasting $y$ by  $F$  under the assumption that $y$ is random with distribution~$F$. This is also known as the entropy function of the score \citep{gneiting2007strictly}. The main contribution of this paper is the introduction and development of new pointwise and SUR data acquisition criteria based on the expected threshold-weighted CRPS. The CRPS has already been used in the realm of GP modeling; however, to the best of our knowledge, its application has been limited to parameter selection \citep[e.g.,][]{petit2023parameter}. \\

We showcase and compare the derived criteria and resulting sequential strategies on chemical test cases revolving around the Photoswitch dataset \citep{thawani2020photoswitch}. In their original study, \citet{thawani2020photoswitch} benchmarked GP models against several alternatives on the Photoswitch dataset, including Random Forests, Attentive Neural Processes, Graph Convolutional Networks, Graph Attention Networks, and Directed Message Passing Neural Networks. While such models primarily target improvements in predictive accuracy within a standard regression framework, the focus of the present work is different: we aim to develop uncertainty-based sequential design strategies that efficiently guide data acquisition. Within this setting, GPs remain the natural modeling choice, as they provide closed-form analytical expressions for our acquisition criteria. This focus is not restrictive, as the best-performing model in the benchmark of \citet{thawani2020photoswitch} was a GP, and GPs consistently demonstrated strong performance in molecular property prediction \citep{deringer2021gaussian, griffiths2024gauche}.\\

Within the GP framework, the choice of kernel heavily depends on the molecule representation used to encode chemical structures. We consider molecules represented by molecular fingerprints \citep{christie1993structure, johnson1990concepts}, which encode the presence or absence of characteristic substructures using binary digits. For fingerprint representations, the Tanimoto kernel \citep{gower1971general} is widely utilized \citep[e.g.,][]{ralaivola2005graphkernels, miyao2019iterative, tripp2024tanimoto}, favored due to its simplicity, efficiency, and effectiveness in measuring the overlap between the presence or absence of molecular features. In addition to Morgan fingerprints, \citet{thawani2020photoswitch} explored alternative molecular representations, and determined that the best-performing combination was the GP model with a Tanimoto kernel and their own hybrid molecular descriptor called fragprints. However, in this work, we continue to use Morgan fingerprints, as they are more broadly known, while exploring a few alternative kernel choices in the supplementary material.\\

By combining the Tanimoto kernel with our targeted CRPS-based criteria, we demonstrate how scoring rules can be leveraged for efficient sequential design towards uncovering molecules of interest. The novel closed-form expressions for CRPS-based acquisition functions presented in this paper are broadly applicable beyond chemistry and open new avenues for advancing sequential design methodologies based on scoring rules.
The structure of the remaining paper is as follows: Section \ref{Sec2} reviews the fundamentals of GP modeling, excursion set estimation and targeted sequential design. While Section \ref{Sec3} is dedicated to the threshold-weighted CRPS and the development of the acquisition criteria based thereon, the numerical tests are introduced in Section \ref{Sec4} and the results are presented in Section \ref{Sec5}. Section \ref{Sec6} wraps up the main body with a discussion and concluding words.

\section{Gaussian Process Modeling and Sequential Design}
\label{Sec2}


\subsection{Gaussian process modeling and excursion set estimation} 
\label{GPmodeling}
Gaussian Processes (GPs) provide a flexible probabilistic framework for modeling unknown functions. A GP $\boldsymbol{\xi}=\left(\xi(\boldsymbol{x})\right)_{\boldsymbol{x} \in \mathcal{X}}$ is a collection of real-valued random variables $\xi(\boldsymbol{x})$ defined on a common probability space, such that for any finite subset $\{\boldsymbol{x}_1, \dots, \boldsymbol{x}_n\} \subset \mathcal{X}$ with $n \geq 1$, the random vector $(\xi(\boldsymbol{x}_1), \dots, \xi(\boldsymbol{x}_n))$ follows a multivariate Gaussian distribution. The distribution of a GP is characterized by its mean $m : \mathcal{X} \to \mathbb{R}$,  $m(\boldsymbol{x}) := \mathbb{E}[\xi(\boldsymbol{x})] $ and covariance function $k : \mathcal{X} \times \mathcal{X} \to \mathbb{R} $, $k(\boldsymbol{x}, \boldsymbol{x}') := Cov (\xi(\boldsymbol{x}), \xi(\boldsymbol{x'}))$.\\

In GP modeling, an unknown function $f: \mathcal{X} \rightarrow \mathbb{R}$ is assumed to be a realization of a GP $\boldsymbol{\xi}$. Here, we assume we observe noisy evaluations of $f$ at \( n \) input locations $\{\boldsymbol{x}_1, \dots, \boldsymbol{x}_n\} \subset \mathcal{X}$, yielding the observations: $ Z_i = f(\boldsymbol{x_i}) + \varepsilon_i, \ \text{where } \varepsilon_i  \sim \mathcal{N}(0, \tau^2) \ \text{i.i.d.} $. Conditioning the GP on the observation event denoted as $ \mathcal{A}_{n} = \{ \boldsymbol{Z}_n = \boldsymbol{z}_n \}$ results in a posterior Gaussian distribution characterized by its predictive mean $m_n(\boldsymbol{x})$ and posterior covariance function  $k_n(\boldsymbol{x}, \boldsymbol{x'})$, whose expressions are recalled in the supplementary material (Part~A).\\

A key advantage of GPs in our context is their ability to support the analysis of excursion sets of the form $\Gamma := \{\boldsymbol{x} \in \mathcal{X}: f(x) \geq \t\}$, where $\t \in \mathbb{R}$ is prescribed. Indeed, we can compute the excursion probability at any location $\boldsymbol{x} \in \mathcal{X}$ in closed form:
\begin{align}
\label{exprob}
        p_n(\boldsymbol{x}) = \mathbb{P} \left( \xi_n(\boldsymbol{x}) \geq \t \right) 
        = \mathbb{P} \left( \frac{\xi_n(\boldsymbol{x}) - m_n(\boldsymbol{x})}{\sqrt{k_n(\boldsymbol{x},\boldsymbol{x})}} \geq \frac{\t - m_n(\boldsymbol{x})}{\sqrt{k_n(\boldsymbol{x},\boldsymbol{x})}} \right) 
        = \Phi \left( \frac{m_n(\boldsymbol{x}) - \t}{\sqrt{k_n(\boldsymbol{x},\boldsymbol{x})}} \right), 
\end{align}
where $\Phi(\cdot)$ is the CDF of the univariate standard Gaussian distribution. \\

The excursion probability can be utilized for classification purposes \citep{bect2012sequential}. Considering a hard classifier $\eta_n: \mathcal{X} \rightarrow \{0,1\}$, we denote the probability of misclassifications by $\tau_n(\boldsymbol{x}) = \mathbb{P}(\eta_n(\boldsymbol{x}) \neq \mathbbm{1}\{ \xi(\boldsymbol{x}) \geq \t \})$. It holds that
$\tau_n(\boldsymbol{x}) = p_n(\boldsymbol{x}) + (1-2 p_n(\boldsymbol{x}))\eta_n(\boldsymbol{x})$, which is minimized if $\eta_n(\boldsymbol{x}) = \mathbbm{1}\{ p_n(\boldsymbol{x}) \geq 0.5 \}$. Using this classifier, we thus arrive at the following estimator for the excursion set \citep{bect2012sequential}:
\begin{equation}
\label{gammahat}
    \hat{\Gamma} := \{\boldsymbol{x} \in \mathcal{X}: p_n(\boldsymbol{x}) \geq 0.5\} = \{\boldsymbol{x} \in \mathcal{X}: m_n(\boldsymbol{x}) \geq \t\}. 
\end{equation}
Depending on the specific objectives of the application, an alternative threshold than $0.5$ might be of interest. If sensitivity is prioritized, a threshold below $0.5$ might be of interest, such that more locations are classified as in the excursion set. Conversely, raising it above $0.5$ shifts the focus on specificity, leading to a more conservative classification \citep{Azzimonti_2019}. \\

GP-based excursion set estimation extends beyond binary classification. The full probabilistic predictions provide richer information, including uncertainty quantification, and can be leveraged in various ways. In particular, this uncertainty can guide active learning strategies through targeted sequential design, ensuring that new evaluations refine the GP model with respect to the excursion set. Additionally, it can be incorporated into scoring-based approaches to evaluate, and ultimately improve, predictions.

\subsection{Targeted Sequential Design}
\label{sequentialdesign}
To refine the GP model with a focus on the excursion set, it is essential to carefully choose added training points. We consider one-step look-ahead (or myopic) strategies, which rely on a similar basic brick: from the `current' conditional distribution of $\boldsymbol{\xi_n}$ given $\mathcal{A}_n$ based on the initial $n$ training points, choose the next point $\boldsymbol{x_{n+1}}$ by optimizing a selection criterion, a.k.a. acquisition function. Subsequently, the noise-contaminated $Z_{n+1}$ is collected and $\mathcal{A}_{n}$ is augmented into $\mathcal{A}_{n+1}$, leading to an update of the GP to $\boldsymbol{\xi_{n+1}}$. There exist a variety of selection criteria for sequential design; we distinguish between pointwise and SUR criteria.  

\paragraph{Pointwise selection criteria}
Given a GP $\boldsymbol{\xi_n} = \left(\xi_n(\boldsymbol{x})\right)_{\boldsymbol{x} \in \mathcal{X}}$, a pointwise selection criterion considers the marginal distributions and selects the next acquisition point \(\boldsymbol{x}_{n+1} \in \mathcal{X}\) by maximizing
\begin{equation}
    \label{pointwise}
     \boldsymbol{x} \mapsto G_n(\boldsymbol{x}) = \mathcal{G}(\xi_n(\boldsymbol{x})),
\end{equation}
where $\mathcal{G}(\cdot)$ maps the marginal Gaussian $\xi_n(\boldsymbol{x}) \sim \mathcal{N}(m_n(\boldsymbol{x}), k_n(\boldsymbol{x}, \boldsymbol{x}))$ to a real value. A very simplistic approach would be to consider the variance $G_n(\boldsymbol{x}) = k_n(\boldsymbol{x}, \boldsymbol{x})$ of all candidate points $\boldsymbol{x} \in \mathcal{X}$ and select one with highest variance as $\boldsymbol{x_{n+1}}$. For targeted sequential design focused on the excursion set $\Gamma$ 
, the objective is to define a local criterion that balances the trade-off between proximity to $\Gamma$ and high predictive uncertainty \citep[][Sections \ref{twCRPS_acquisition_pw} and \ref{alternatives_acquisition}]{chevalier2014fast}. 

\paragraph{SUR selection criteria}
Stepwise Uncertainty Reduction (SUR; \citeauthor{bect2019supermartingale}~\citeyear{bect2019supermartingale}) criteria build upon notions of  uncertainty that may involve the GP model on the entire domain rather than at a single point, and consider how such a measure of uncertainty evolves at time $n + 1$ if the next evaluation is conducted at $\boldsymbol{x_{n+1}} = \boldsymbol{x}$. Given a GP model $\boldsymbol{\xi_n}$ at time~$n$, a SUR criterion selects the next acquisition point by minimizing:
\begin{equation}
    \label{SUR}  
    \boldsymbol{x} \mapsto J_n(\boldsymbol{x}) = \mathbb{E}_{n, \boldsymbol{x}}[\mathcal{H}(\GPupdate)],
\end{equation}
where $\GPupdate \sim \mathrm{GP}(\GPupdateM, \GPupdateK)$ is the future GP conditioned on $\GPupdateO$. Here $\mathcal{H}(\cdot)$ stands for an uncertainty functional that maps the whole (updated) GP model to a real-valued variable that summarizes the remaining uncertainty. Furthermore, $\mathbb{E}_{n, \boldsymbol{x}}(\cdot)$ denotes the conditional expectation with respect to $\boldsymbol{\xi_n}$ when  $\boldsymbol{x_{n+1}} = \boldsymbol{x}$. For a simplistic acquisition approach based on variance, we could employ the integrated future variance $J_n(\boldsymbol{x}) = \int \GPupdateK (\boldsymbol{x'}, \boldsymbol{x'}) \mathbb{P}_{\mathcal{X}}(\mathrm{d}\boldsymbol{x'})$. SUR criteria targeting response-dependent regions are discussed in Sections \ref{twCRPS_acquisition_int} and~\ref{alternatives_acquisition}. 

\section{Threshold-weighted CRPS for Sequential Design}
\label{Sec3}

\subsection{Threshold-weighted CRPS}
\label{twcrps_section}

Scoring rules \citep{gneiting2007strictly} provide a tool for evaluating probabilistic forecasts by measuring the alignment between a predictive distribution and the materialized value. One very popular scoring rule is the Continuous Ranked Probability Score \citep[CRPS;][]{matheson1976scoring}, which for a predictive distribution with CDF $F(\cdot)$ and a materialized value $y$ is defined as,
    \begin{equation}
     \label{CRPS}
    \mathrm{CRPS}(F, y) = \int_{-\infty}^{\infty} [F(u) - \mathbbm{1}\{y \leq u\}]^2 \, \mathrm{d}u. 
    \end{equation}

In GP modeling, we would like to evaluate the performance of $\xi_n(\boldsymbol{x'})\sim \mathcal{N}\left(m_n(\boldsymbol{x'}), k_n(\boldsymbol{x'},\boldsymbol{x'}) \right)$ in point $\boldsymbol{x'} \in \mathcal{X}$ with respect to $f(\boldsymbol{x'})$. We only consider one point $\boldsymbol{x'}$ at a time (univariate), for which we use $\GPpredictCDF$ to denote the Gaussian predictive distribution with mean $\GPpredictmean = m_n(\boldsymbol{x'})$ and variance $\GPpredictvar = k_n(\boldsymbol{x', \boldsymbol{x'}})$. As in the traditional setting, we use $\GPobs = f(\boldsymbol{x'})$ to denote the materialized value. In this setting, the CRPS writes as \citep{gneiting2007strictly}, 
    \begin{equation}
        \mathrm{CRPS}\left(\GPpredictCDF, \GPobs \right) = \GPpredictsd \left[  
        \widetilde{\GPobs} \left( 2 \Phi\left( \widetilde{\GPobs} \right) -1 \right)
        + 2 \phi \left( \widetilde{\GPobs} \right) - \frac{1}{\sqrt{\pi}},
        \right].
    \end{equation}
using $\widetilde{\GPobs} = (\GPobs - \GPpredictmean)/\GPpredictsd$. The CRPS is negatively oriented, meaning that better predictions correspond to lower scores. \\
    
Weighted scoring rules enable a targeted assessment of forecasts. The threshold-weighted CRPS \citep{matheson1976scoring, gneiting2011comparing} extends the CRPS by considering a weighting measure $\gamma$, which emphasizes the values of interest:
    \begin{equation}
    \label{twcrps_def}
    \twcrps(F, y) = \int_{-\infty}^{\infty} [F(u) - \mathbbm{1}\{y \leq u\}]^2 \, \intB. 
    \end{equation}

While $\gamma$ may potentially be taken as any Borel measure on $\mathbb{R}$, in our excursion setting, 
we consider two types of weighting measures (characterized by their densities with respect to the Lebesgue measure $\lambda$). The first weighting measure $\wI$ only considers values above the threshold $\t$ and is characterized by $\intI = \mathbbm{1}\{\t \leq u\} \mathrm{d} \lambda(u)$. The second weighting measure $\wP$ admits a Gaussian probability density centered on the threshold $\t$ and with variance $\sigmaY^2 >0$: $\intP = \frac{1}{\sigmaY} \phi \left( \frac{u - \t}{\sigmaY} \right)  \mathrm{d} \lambda(u)$, where $\phi: \mathbb{R} \rightarrow (0,\infty), \phi(v):=\frac{1}{\sqrt{2\pi}}\exp^{-\frac{v^2}{2}}$.  

\begin{theorem}\label{scoring_twcrps}
    For a predictive distribution $\GPpredictCDF$, 
    it holds with weighting measure $wI$ that 
        \begin{align*}
            \seqIp(\GPpredictCDF, \GPobs) 
            = \GPpredictsd \left[ -\widetilde{\t} \Phi(\widetilde{\t})^2
            + \widetilde{y^{\t}} \left( 2 \Phi(\widetilde{y^{\t}}) - 1  \right) +  \left( 2 \phi \left( \widetilde{y^{\t}} \right) - 2 \phi \left( \widetilde{\t} \right) \Phi(\widetilde{\t}) \right) -  \frac{1}{\sqrt{\pi}} \left( 1 - \Phi(\widetilde{\t} \sqrt{2}) \right) \right],
        \end{align*}
        where $\widetilde{\t} = (\t - \GPpredictmean)/\GPpredictsd$ and $\widetilde{\GPobs^t} = \Big(\max(\GPobs,\t) - \GPpredictmean \Big)/\GPpredictsd$. 
        
        For the weighting measure $\wP$ (with mean $t$ and variance $\sigmaY^2$), one obtains
        \begin{align*}
            \seqPp(\GPpredictCDF, \GPobs) 
            &= \Phi \left( \begin{pmatrix} 0 \\ 0 \end{pmatrix}; \begin{pmatrix} \GPpredictmean - \t \\ \GPpredictmean - \t \end{pmatrix},  
            \begin{pmatrix} \sigmaY^2 + \GPpredictvar & \sigmaY^2 \\ \sigmaY^2 & \sigmaY^2 + \GPpredictvar \end{pmatrix}  \right) \\
            &\quad \quad -2  
            \Phi \left( \begin{pmatrix} 0 \\ 0 \end{pmatrix}; \begin{pmatrix} \GPpredictmean - \t \\ y - \t \end{pmatrix},  
            \begin{pmatrix} \sigmaY^2 + \GPpredictvar & \sigmaY^2 \\ \sigmaY^2 & \sigmaY^2  \end{pmatrix}  \right) + \Phi \left( \frac{\t - \GPobs}{\sigmaY} \right),
        \end{align*}
        where $\Phi(\cdot; \boldsymbol{\mu}, \boldsymbol{\Sigma})$ is the bivariate Gaussian CDF with parameters $\boldsymbol{\mu} \in \mathbb{R}^2$ and $\boldsymbol{\Sigma} \in \mathbb{R}^{2 \times 2}$.
\end{theorem}

\begin{proof}
    See supplementary material (Part~C.1).
\end{proof}

For illustration purposes, Figure \ref{fig:illu_twcrps} depicts the discussed versions of the CRPS for a Gaussian $\GPpredictCDF$ (solid line), showing the area under the integral for the CRPS (a and d), $\seqIp$ (b and e), and $\seqPp$ (c and f) across two scenarios: In the first row (a-c) we consider $y < \t < \GPpredictmean$, referring to a point which does not lie in the excursion set ($y < \t$), but is falsely predicted to be in the excursion set ($\GPpredictmean > \t$, see Eq.~\ref{gammahat}). In the second row (d-f), we use $\GPpredictmean < \t < y$, such that the point does lie in the excursion set ($y > \t$), but is falsely predicted to be not ($\GPpredictmean < \t$). Notably, both the (unweighted) CRPS and $\seqPp$ are symmetric, yielding the same value for both scenarios. Contrarily, the $\seqIp$ has a higher value in the second scenario. 

\begin{figure}
    \centering
    \includegraphics[width=0.95\linewidth]{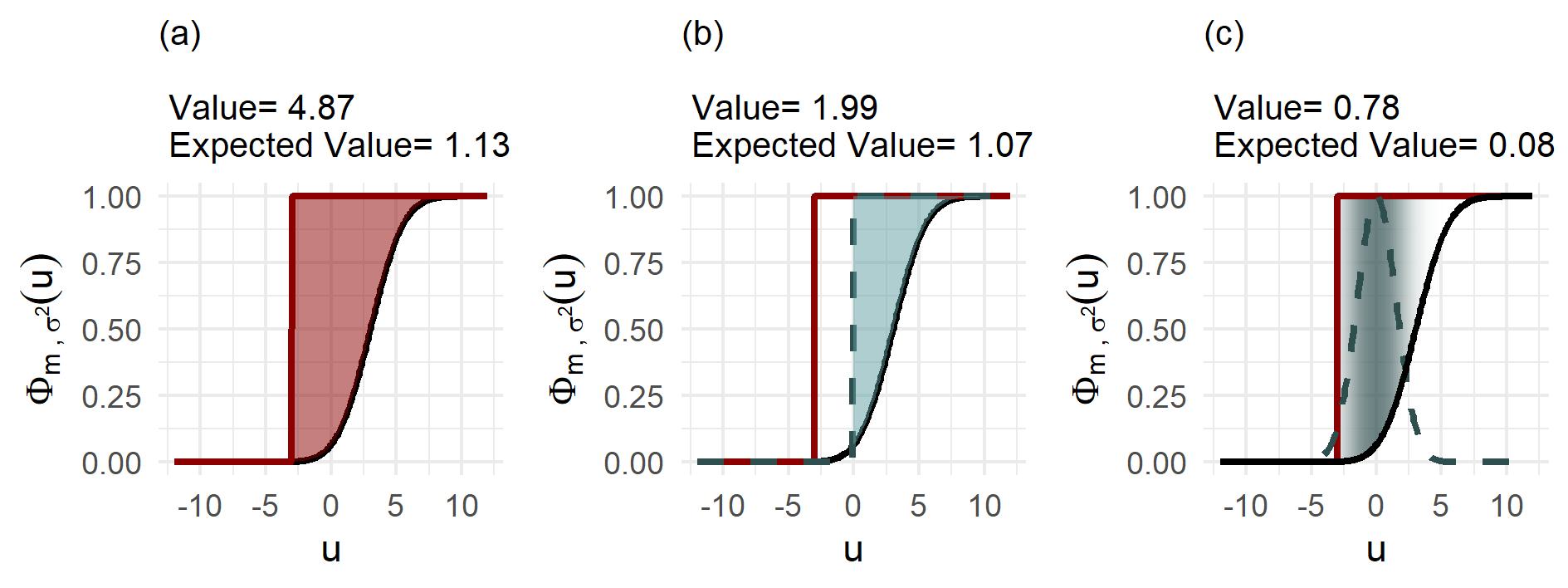} 
    \includegraphics[width=0.95\linewidth]{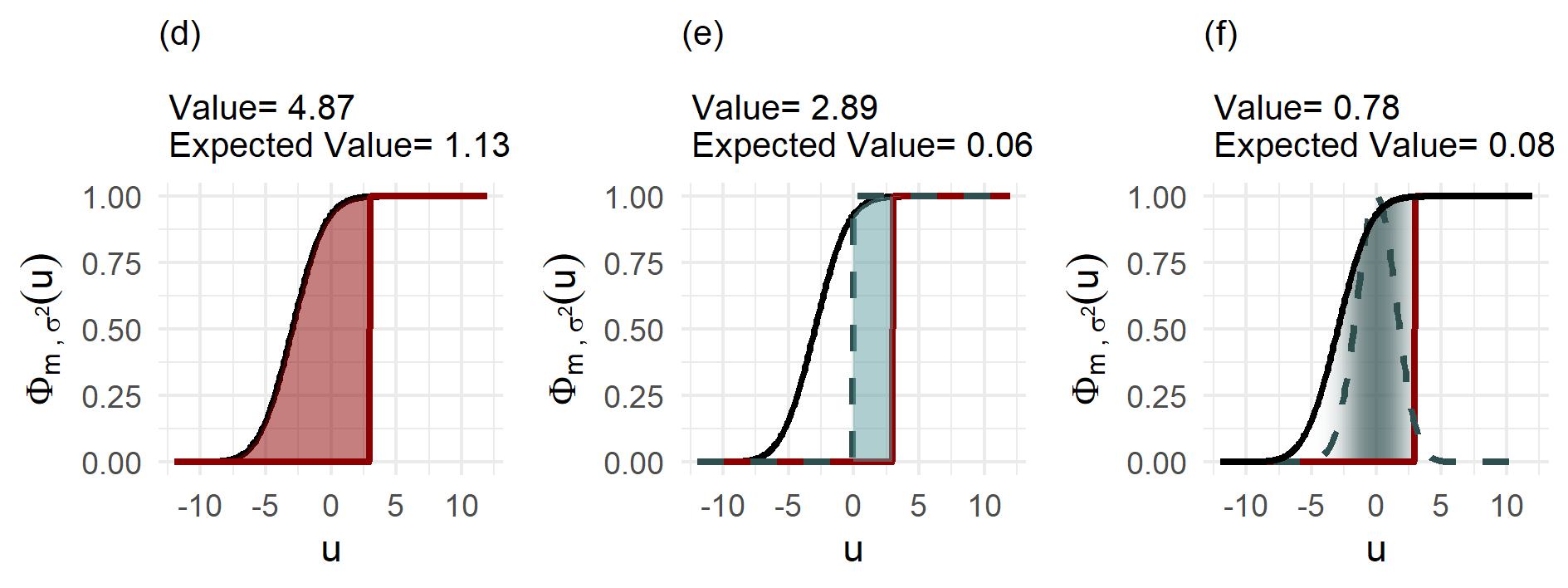} 
    \caption{Illustrations of the $\mathrm{CRPS}\left(\GPpredictCDF, \GPobs \right)$ and $\twcrps\left(\GPpredictCDF, \GPobs \right)$ for a Gaussian CDF $\GPpredictCDF(u)$ (solid line), an observation $y$ (indicator line) and a threshold $\t$. In the first row (a-c) we have $y = -3, \t=0, \GPpredictmean = 3$ and in the second row (d-f) we have $\GPpredictmean=-3, \t=0, y = 3$; we use $\GPpredictsd = 2$ for the predictive Gaussian. The individual plots display the area under the integral for the CRPS (a and d), $\seqIp$ (b and e) and $\seqPp$ (c and f; $\sigmaY = 1.5$), with the dashed lines indicating the weighting measures $\wI$ and $\wP$. The numbers refer to the value and expected value of the respective criterion. }
    \label{fig:illu_twcrps}
\end{figure}

\subsection{Pointwise selection criteria}
\label{twCRPS_acquisition_pw}

Given that the threshold-weighted CRPS evaluates both accuracy and uncertainty of the GP model with respect to the excursion set of interest, it is a promising basis for an acquisition criterion. For a pointwise criterion, we consider the threshold-weighted CRPS of all the candidate points, and add the one to the training set which has the highest score~(Eq.~\ref{pointwise}). However, in its original form for evaluation (Eq.~\ref{twcrps_def}), the threshold-weighted CRPS relies on the materialized value $\GPobs$ in the candidate point, which is unavailable during data acquisition. Therefore, we consider an expected version, where the expectation is taken with respect to the current predictive distribution. This is also known as the expected score or entropy function for proper scoring rules \citep{gneiting2007strictly}. The expected threshold-weighted CRPS with weighting measure $\gamma$ for predictive CDF $F$ is defined as,
    \begin{align}
    \label{EtwCRPS}
    \twcrps(F) 
    = \int_{-\infty}^{\infty} \twcrps(F, y') \mathrm{d}F(y'). 
    \end{align}
The following theorem provides an alternative representation of this expected threshold-weighted CRPS and derives explicit formulas for Gaussian predictive distributions with weighting measures $\wI$ and $\wP$.

\begin{theorem}\label{pointwise_twcrps}
For a predictive CDF $F(\cdot)$ and a weighting measure $\gamma$, it holds,
\begin{equation*}
   \twcrps(F) = \int F(u) \left( 1-F(u) \right)  \intB.   
\end{equation*}
Particularly, for $F = \GPpredictCDF$ and $\intI = \mathbbm{1}\{\t \leq u\}\mathrm{d} \lambda(u)$, one obtains,
\begin{equation*}
    \seqIp(\GPpredictCDF) = \GPpredictsd \Big[ \widetilde{\t} \Phi(\widetilde{\t})^2 - \widetilde{\t}\Phi(\widetilde{\t})  + \frac{1}{\sqrt{\pi}} - \frac{1}{\sqrt{\pi}} \Phi(\widetilde{\t} \sqrt{2}) + 2 \phi(\widetilde{\t}) \Phi(\widetilde{\t}) - \phi(\widetilde{\t}) \Big],
\end{equation*}
with $\widetilde{\t} = (\t - \GPpredictmean)/\GPpredictsd$. An alternative representation is, 
\begin{equation*}
    \seqIp(\GPpredictCDF) = 
    \GPpredictsd \mathbb{E} \left[ \left( \widetilde{N} - \max \left( N, \frac{\t-\GPpredictmean}{\GPpredictsd} \right) \right)_+ \right],
\end{equation*}
for $N, \widetilde{N} \overset{i.i.d.}{\sim}\mathcal{N}(0,1)$. For $\intP = \frac{1}{\sigmaY}\phi \left( \frac{u-t}{\sigmaY} \right)\mathrm{d} \lambda(u)$, we get, 
\begin{align*}
    \seqPp(\GPpredictCDF) &= \Phi \left( \begin{pmatrix} 0 \\ 0 \end{pmatrix}; \begin{pmatrix}  \GPpredictmean - \t \\ \t -\GPpredictmean \end{pmatrix},  
    \begin{pmatrix} \sigmaY^2 + \GPpredictvar & -\sigmaY^2 \\ -\sigmaY^2 & \sigmaY^2 + \GPpredictvar \end{pmatrix}  \right)  .
\end{align*}
\end{theorem}

\begin{proof}
    Supplementary material (Part~C.2); the first result generalizes the corresponding expression for the traditional CRPS \citep{gneiting2007strictly}.
\end{proof}

To use this for pointwise selection criteria (Eq.~\ref{pointwise}), we can set $G_n(\boldsymbol{x}) = \twcrps(\GPpredictCDF)$, where $\GPpredictmean = m_n(\boldsymbol{x})$ and $\GPpredictvar = k_n(\boldsymbol{x, \boldsymbol{x}})$. To better understand the resulting behavior, we return to Figure \ref{fig:illu_twcrps}, in which we also denote the expected score values. Both CRPS and $\seqPp$ maintain symmetry in their expected scores (a, c, d, f), while $\seqIp$ remains asymmetric but with a reversed orientation. In the original scoring rule relying on $y$, a missed point within the excursion set results in a higher score. In the expected score, a point mistakenly classified as being within the excursion set yields a higher expected score. In general, the $\seqIp$ is higher for cases where $\GPpredictCDF$ has $m > \t$, regardless of whether the prediction is correct (since the correct outcome is unknown). Figure \ref{fig:illu_twcrps2}a shows this relationship for $\GPpredictCDF$ with a fixed $\t = 0$. Both the values of $\seqIp$ and $\seqPp$ increase with higher uncertainty ($\GPpredictsd$) for the same $\GPpredictmean$. However, while $\seqIp$ continuously increases with rising $\GPpredictmean$ when $\GPpredictsd$ is fixed, $\seqPp$ remains symmetric around the threshold $\t$ for varying $\GPpredictmean$. \\

\begin{figure}
    \centering
    \includegraphics{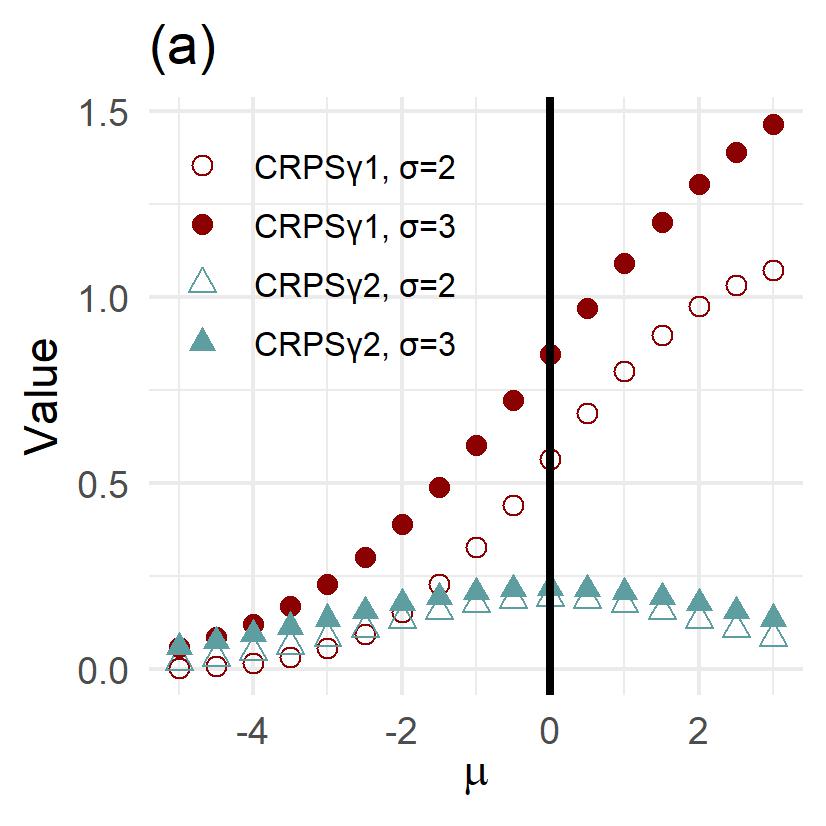} \hspace{.7cm}
    \includegraphics{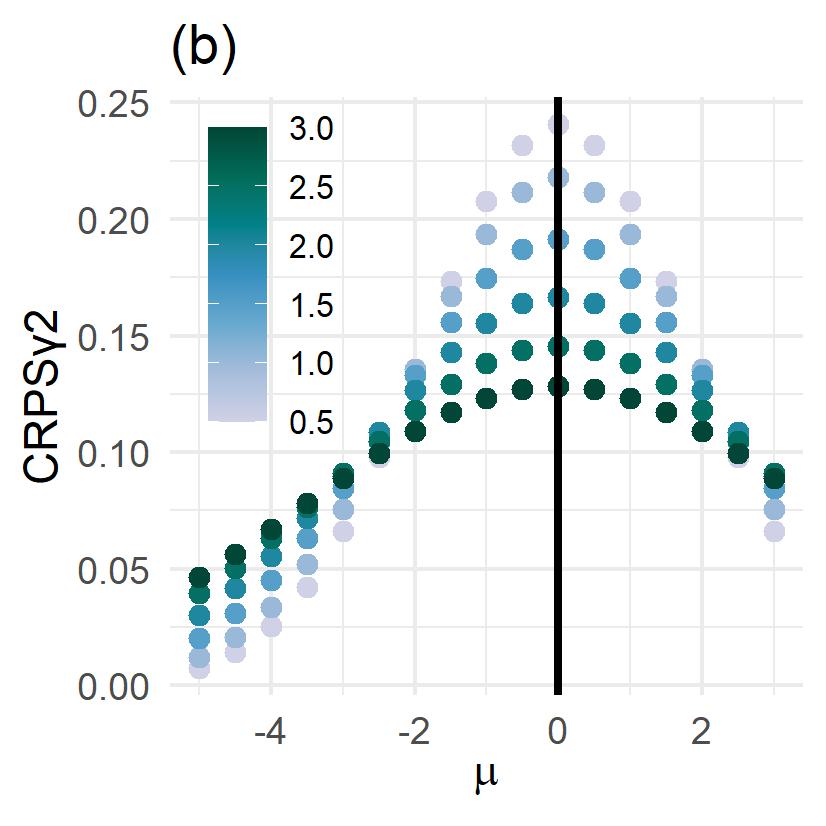} 
    \caption{Illustrations of the expected $\seqIp$ and $\seqPp$ for a Gaussian~$\GPpredictCDF$ depending on the predictive mean $\GPpredictmean$: (a) Expected $\seqIp$ and $\seqPp$ ($\sigmaY = 1.5$) and (b) Expected $\seqPp$ for different values of $\sigmaY$ ($\GPpredictsd = 2$). In both plots, the vertical solid line indicate the threshold $\t=0$. }
    \label{fig:illu_twcrps2}
\end{figure}

While the indicator density of $\seqIp$ is fixed (although a shift would be thinkable), the density of $\seqPp$ is influenced by the parameter $\sigmaY$, which controls the width of the weight around the threshold $\t$. Figure \ref{fig:illu_twcrps2}b illustrates the dependence of $\seqPp$ on the predictive mean $\GPpredictmean$ for different values of $\sigmaY$ with $\t = 0$. A lower value of $\sigmaY$ places greater emphasis on the region near the threshold, while a higher value results in more uniform values across different $\GPpredictmean$.

\subsection{SUR selection criteria}
\label{twCRPS_acquisition_int}

In pointwise selection, we consider the marginal Gaussians, whereas a SUR criterion takes into account the uncertainty of the whole GP. Relying on the expected threshold-weighted CRPS, we propose the following SUR functional (Eq.~\ref{SUR}) and select $\boldsymbol{x_{n+1}}$ to minimize:
\begin{align}
    J_n(\boldsymbol{x}) &= \mathbb{E}_{n, \boldsymbol{x}}[\mathcal{H}(\GPupdate)] \\
    &=  \mathbb{E}_{n, \boldsymbol{x}} \left[ \int \twcrps(\GPupdateU (\boldsymbol{x'}))   \mathbb{P}_{\mathcal{X}}(\mathrm{d}\boldsymbol{x'}) \right] \nonumber \\
    &=  \int \mathbb{E}_{n, \boldsymbol{x}} \left[  \twcrps(\GPupdateU(\boldsymbol{x'}))  \right] \mathbb{P}_{\mathcal{X}}(\mathrm{d}\boldsymbol{x'}). \nonumber
\end{align}
This means we want to select $\boldsymbol{x_{n+1}}$ such that it minimizes the expected integrated future $\twcrps(\GPupdate)$. We continue by deriving expressions for the integrand, keeping $\boldsymbol{x}$ and $\boldsymbol{x'}$ fixed. In doing so, we consider the marginal Gaussian $\GPupdateU(\boldsymbol{x'}) \sim \mathcal{N}(\GPupdateM(\boldsymbol{x'}), \GPupdateK(\boldsymbol{x'}, \boldsymbol{x'}))$. While the updated variance $\GPupdateK(\boldsymbol{x'}, \boldsymbol{x'})$ only depends on $ \boldsymbol{x_{n+1}} = \boldsymbol{x}$, the updated mean $\GPupdateM(\boldsymbol{x'})$ also depends on the noise-contaminated corresponding value $z_{n+1} = f(\boldsymbol{x_{n+1}}) + \epsilon_{n+1}$. When we consider adding $\boldsymbol{x_{n+1}} = \boldsymbol{x}$, we neither can nor want to evaluate all candidate points in advance, so we do not yet know $z_{n+1}$. However, using the expectation $\mathbb{E}_{n,\boldsymbol{x}}$ with respect to $\boldsymbol{\xi_n}$, we can write $Z_{n+1} = m_n(\boldsymbol{x}) + \sqrt{k_n(\boldsymbol{x}, \boldsymbol{x}) + \tau^2} V$ with $V \sim \mathcal{N}(0,1)$ such that
\begin{align}
    \GPupdateM(\boldsymbol{x'}) | V
    = m_n(\boldsymbol{x'}) + \alpha_{n}(\boldsymbol{x'}) V, \quad \text{with }\alpha_{n}(\boldsymbol{x'}) = \frac{k_n(\boldsymbol{x}, \boldsymbol{x'})}{\sqrt{k_n(\boldsymbol{x}, \boldsymbol{x}) + \tau^2}}.
\end{align}
This leads to the following theorem:

\begin{theorem}\label{sur_twcrps}
        We consider the conditional GP $\GPupdate \sim \mathrm{GP}(\GPupdateM, \GPupdateK)$ and its predecessor $\boldsymbol{\xi_n} \sim \mathrm{GP}(m_n, k_n)$. For fixed $\boldsymbol{x}, \boldsymbol{x'} \in \mathcal{X}$ and using $\sigma^2_{n+1} = \GPupdateK(\boldsymbol{x'}, \boldsymbol{x'})$, $\mu_n = m_n(\boldsymbol{x'})$, $\alpha_n = \alpha_{n}(\boldsymbol{x'})$, we obtain for $\wI$: 
        \begin{align*}
        &\mathbb{E}_{n,\boldsymbol{x}} \left[  \seqIp(\GPupdateU(\boldsymbol{x'}))  \right] 
        = \sigma_{n+1} \mathbb{E} \left[ \left( N - \max \left(\widetilde{N}, \frac{\t - \mu_n - \alpha_{n} V}{\sigma_{n+1}} \}   \right) \right)_+ \right],
        \end{align*}
        with $N, \widetilde{N}, V \overset{i.i.d.}{\sim} \mathcal{N}(0,1)$.
       For $\wP$, one derives,
        \begin{align*}
        &\mathbb{E}_{n,\boldsymbol{x}} \left[  \seqPp(\GPupdateU(\boldsymbol{x'}))  \right] &= \Phi \left( \begin{pmatrix} 0 \\ 0 \end{pmatrix}; \begin{pmatrix}  \mu_n - t \\ \t - \mu_n \end{pmatrix},  
            \begin{pmatrix} 
            \sigma_{n+1}^2 + \alpha_{n}^2 + \sigmaY^2 
            & -\alpha_{n}^2 - \sigmaY^2 \\ 
            -\alpha_{n}^2 - \sigmaY^2 
            & \sigma_{n+1}^2 + \alpha_{n}^2 + \sigmaY^2
            \end{pmatrix}  \right)  .
        \end{align*}
\end{theorem}

\begin{proof}
    Supplementary material (Part~C.3).
\end{proof}

Using these formulas for the integrand, the SUR criteria $J_n(\boldsymbol{x})$ can be calculated by integrating over $\boldsymbol{x'} \in \mathcal{X}$. In a discrete molecular space, this corresponds to taking the mean of the molecules. We refer to the SUR criteria, which are based on integration, with $\seqIs$ and $\seqPs$. \\

To ensure the theoretical soundness of these SUR criteria, we establish in Part~D of the supplementary material that $\seqIs$ and $\seqPs$ are consistent in the sense of \citet{bect2019supermartingale}. Specifically, we show that the corresponding uncertainty measures $\seqIp$ and $\seqPp$ converge almost surely to zero as the number of observations increases, thereby ensuring the asymptotic correctness of the sequential design.

\section{Test cases, implementation and evaluation}
\label{Sec4}


\subsection{Test cases}
\label{test}

Both of our test cases revolve around the Photoswitch dataset \citep{thawani2020photoswitch}. Photoswitches represent a category of molecules capable of undergoing a reversible transformation between various structural states when exposed to light. Their versatility extends to applications in fields such as medicine and renewable energy, where their effectiveness hinges on their electronic transition wavelength. Our dataset comprises 392 molecules, each featuring the experimentally determined transition wavelength of its E isomer. We represent the molecules by Morgan fingerprints with a length of 2048 bits and a radius of~3. \\

The original Photoswitch dataset only provides noisy observations, meaning that the ground truth of the function $f$ and the excursion set $\Gamma$ are unavailable. Before using this 'real' test case, we also construct a synthetic version to provide a more controlled benchmark for the new acquisition criteria. To generate the synthetic version of the dataset, we fit a GP using all the molecules. Then, we assume that the resulting $m_N$ represents the ground truth $f$. To generate the noisy observations $Z_i$, we add i.i.d. Gaussian noise, with the variance determined by the GP fit. We emphasize that the synthetic version of the dataset ensures that the data corresponds to the GP model, which may not necessarily be the case in practice. 

\subsection{Tanimoto Kernel}
\label{tanimotokernel}


Following \cite{thawani2020photoswitch}, we use fingerprints as input representations. Given the nature of these high-dimensional binary inputs, we not only consider popular traditional kernels (as presented in the supplementary material), but also kernels that are specialized for such non-continuous input spaces. For a comprehensive comparison of kernels, including alternative molecular representations, we refer to \citet{griffiths2024gauche}.\\

The Tanimoto kernel, originally introduced by \citet{gower1971general} as a similarity measure for binary attributes, is well suited to our setting.  For $\boldsymbol{x},\boldsymbol{x'} \in \{0,1\}^d$, $d \geq 1$ it can be defined as,
\begin{equation}
\label{tanimoto}
    k_{\mathrm{TAN}}(\boldsymbol{x}, \boldsymbol{x'}) := \left\{ \begin{array}{ll} \sigma_{k}^2 & \text{ if } \langle \boldsymbol{x},  \boldsymbol{x}\rangle = \langle \boldsymbol{x'},  \boldsymbol{x'} \rangle = 0, \\ \sigma_{k}^2 \frac{\langle \boldsymbol{x},  \boldsymbol{x'} \rangle}{||\boldsymbol{x}||^2 + ||\boldsymbol{x'}||^2 - \langle \boldsymbol{x},  \boldsymbol{x'} \rangle } & \text{ else, } \end{array} \right.
\end{equation}
where $\langle \boldsymbol{x},  \boldsymbol{x}\rangle = \boldsymbol{x}^\top \boldsymbol{x}$ and $\sigma_{k}^2 > 0$ is a variance parameter. If the binary elements of the vectors stand for the presence and absence of features or properties of an object, the Tanimoto kernel is counting the number of common features normalized by the total number of features present. It is also known as Jaccard's index \citep{jaccard1901etude} in the literature.

\subsection{Implementation and configuration}
\label{implementation}
The code for this study can be found on \href{https://github.com/LeaFrie/GP_sequentialCRPS}{GitHub}, and the original Photoswitch dataset is publicly available (see e.g., \href{ https://github.com/Ryan-Rhys/The-Photoswitch-Dataset/tree/master/dataset}{Photoswitch dataset}).
For implementing the GP models, we utilized the R-package \textit{kergp} \citep{kergp}. This is a package for GP regression where some pre-defined kernels are available, and it is also possible to define customized kernels through a formula mechanism. The latter was employed in the case of the Tanimoto kernel. We consider the variance term of the kernel (such as the integral scale if there is one) as hyperparameters, estimated simultaneously using maximum likelihood estimation with numerical optimization via gradient based methods. In the test case concerning the original data with the noise being unknown, we consider its variance as an additional hyperparameter. For the sequential data acquisition, we implemented the different criteria manually. For the integral criteria, we calculate the integrals by averaging over the candidate set. Furthermore, for $\seqIs$, we employ Monte Carlo sampling of the standard-normally distributed random variables. \\

The synthetic test case is constructed using a Tanimoto kernel, so we naturally also use it in the later numerical experiments. For the choice of kernel used on the original dataset, we compared the performance of the Tanimoto, Gaussian and Exponential kernels using $10,20,30\%$ of the datasets for training (supplementary material, Part~B). Given the consistently good performance of the Tanimoto kernel, even with small training datasets, we continue focusing on this kernel, also for the original dataset. For the threshold $\t$ specifying the excursion set, we employ the empirical 0.8-quantile of the distribution of the numerical response of interest to ensure that we have a reasonable number of molecules with the desired property in the dataset.

\subsection{Performance evaluation}
\label{evaluation_metrics}

To prevent a scenario where all excursion set molecules are sequentially added to the training set, leaving no examples in the test set, we first reserve a subset of molecules for validation in each trial. This leads to a split of the data into three sets: the initial training set with points $\boldsymbol{x_1}, \dots, \boldsymbol{x_n}$, the candidate training set $\boldsymbol{x_{n+1}}, \dots, \boldsymbol{x_m}$ and the validation set with points $\boldsymbol{x_{m+1}}, \dots, \boldsymbol{x_N}$. We specify the number of validation points as $M=N-m = 100$ in both test cases. Eventually, the predictive performance of a model is evaluated on the validation set by considering the mean metric of the validation points. \\

In the test case which we constructed such that the ground truth is known, we consider the following metrics: Naturally, we use the CRPS-based evaluation metrics (Section \ref{Sec3}), wherein we compare the latent GP $\GPpredictCDF$ with mean $\GPpredictmean = m_n(\boldsymbol{x})$ and variance $\GPpredictvar = k_n(\boldsymbol{x, \boldsymbol{x}})$ with the ground truth $y = f(\boldsymbol{x})$. To avoid confusion, we denote the evaluation metrics as $\evaB$, $\evaI$ and $\evaP$, indicating that we consider the mean value over the validation molecules. Furthermore, we consider the precision and sensitivity of the model given by,
\begin{equation}
\label{rates}
    \mathrm{Sensitivity} = \frac{\mathrm{TP}}{\mathrm{TP} + \mathrm{FN}}, \quad 
    \mathrm{Precision} = \frac{\mathrm{TP}}{\mathrm{TP} + \mathrm{FP}},
\end{equation}
where we use the numbers of true positives (TP; $| \hat{\Gamma} \cap \Gamma |$), false positives (FP; $|\hat{\Gamma} \cap \Gamma^c|$), true negatives (TN; $|\hat{\Gamma}^c \cap \Gamma^c|$) and false negatives (FN; $|\hat{\Gamma}^c \cap \Gamma|$), with $|A|$ denoting the cardinality of set $A$ and $A^c$ representing its complement. Additionally, we consider the Root Mean Squared Error (RMSE), which should not be confused with the acquisition criterion based on the Mean Squared Error (kriging variance). Besides the basic version, we also employ two versions targeting the excursion scenario. The first version exclusively considers molecules actually lying in the excursion set,
\begin{align}
\label{RMSE}
    &\mathrm{RMSE_{\Gamma}} = \sqrt{\frac{1}{M_{\Gamma}}\sum_{i=m+1}^N(m_n(\boldsymbol{x_i}) - f(\boldsymbol{x_i}))^2 \, \mathbbm{1} \{\boldsymbol{x_i} \in \Gamma \} },
\end{align}
with $M_{\Gamma} = \sum_{i=m+1}^N \mathbbm{1} \{\boldsymbol{x_i} \in \Gamma \}$. The second version only accounts for molecules, which are classified by the GP to be in the excursion set,
\begin{align}
    &\mathrm{RMSE_{p}} = \sqrt{\frac{1}{M_p} \sum_{i=m+1}^N(m_n(\boldsymbol{x_i}) - f(\boldsymbol{x_i}))^2 \,  \mathbbm{1} \{\boldsymbol{x_i} \in \hat{\Gamma} \}},
\end{align}
where $M_{p} = \sum_{i=m+1}^N \mathbbm{1} \{\boldsymbol{x_i} \in \hat{\Gamma} \}$. Please note that $\mathrm{RMSE_{p}}$ equals zero for a model with zero sensitivity and, therefore, should only be used in conjunction with a sensitivity measure. While $\mathrm{RMSE_{\Gamma}}$ represents a version of RMSE linked to sensitivity and the performance on the actual excursion set, $\mathrm{RMSE_{p}}$ focuses on precision by evaluating the model's performance specifically for molecules classified by the GP to be within the excursion set. \\

In the test case using the original dataset, the ground truth for $f$ and $\Gamma$ is unknown, making classification less straightforward. Consequently, we focus exclusively on the CRPS-based evaluation metrics outlined in Section \ref{Sec3}. The GP $\boldsymbol{\xi_n}$ acts as a predictor for the (noiseless) function $f$. To evaluate it in the presence of noisy data, we employ the observational GP model, where the covariance function is extended to $k_n + \tau^2$.

\subsection{Alternative selection criteria}
\label{alternatives_acquisition}
The Targeted Mean Square Error (TMSE) criterion, introduced by \citet{picheny2010adaptive}, is a pointwise selection criterion and aims to minimize the kriging variance (Mean Square Error), particularly in regions where $m_n$ closely approximates the threshold $\t$. It relies on a Gaussian PDF around the threshold and aims to maximize,
\begin{equation}
    \label{tmse}
    G_n(\boldsymbol{x}) = k_n(\boldsymbol{x},\boldsymbol{x}) \frac{1}{\sqrt{k_n(\boldsymbol{x},\boldsymbol{x}) + \zeta^2}}\phi \left( \frac{m_n(\boldsymbol{x}) - \t}{\sqrt{k_n(\boldsymbol{x},\boldsymbol{x}) + \zeta^2}} \right)
\end{equation}
with $\zeta \geq 0$ being a parameter tuning the bandwidth of a window of interest around the threshold~$\t$ (here we use $\zeta=0$). As a second pointwise criterion, we consider the entropy-based criterion discussed in \citet{cole2023entropy}. It takes into account the excursion probability $p_n(\boldsymbol{x})$ and focuses on,
\begin{equation}
    \label{entropy}
    G_n(\boldsymbol{x}) = - p_n(\boldsymbol{x}) \log (p_n(\boldsymbol{x})) - (1-p_n(\boldsymbol{x})) \log (1-p_n(\boldsymbol{x})).
\end{equation}

For classification problems, margin-based active learning approaches rely on the distance of inputs to the decision boundary \citep{tong2001support, brinker2003incorporating, balcan2007margin, ducoffe2018adversarial}. To refine the current model, these methods query unlabeled inputs that lie close to the decision boundary, as they are expected to provide the most information for improving the classifier. In the case of binary classification, the entropy-based acquisition criterion can be viewed as related to margin-based active learning approaches, with the exceedance probability reflecting the closeness of an input to the decision boundary. \\

As a first alternative SUR criterion, we consider the Targeted Integrated Mean Square Error (TIMSE; e.g., \citeauthor{picheny2010adaptive}, \citeyear{picheny2010adaptive}), which can be considered as the integral version of TMSE (Eq.~\ref{tmse}). The TIMSE for point $\boldsymbol{x}$ aims to minimize,
\begin{equation}
\label{timse}
    J_n(\boldsymbol{x}) =  \int_{\mathcal{X}} \GPupdateK(\boldsymbol{x'},\boldsymbol{x'})W_n(\boldsymbol{x'}) \mathbb{P}_{\mathcal{X}}(\mathrm{d}\boldsymbol{x'}),
\end{equation}
with $\GPupdateK(\boldsymbol{x'},\boldsymbol{x'})$ denoting the kriging variance at point $\boldsymbol{x'}$ once point $\boldsymbol{x_{n+1}}=\boldsymbol{x}$ has been added. Furthermore, $W_n(\boldsymbol{x'})$ is a weight function in accordance with the definition of the TMSE (Eq.~\ref{tmse}, we use $\zeta=0$). The second SUR criterion under consideration is again guided by the principle that our objective is to achieve clear classification for each point $\boldsymbol{x} \in \mathcal{X}$, with $p_n(\boldsymbol{x})$ taking values of either 0 or 1. Hence, we adopt the approach proposed by \citet{bect2012sequential}, referred to as Integrated Bernoulli Variance (IBV; \citeauthor{bect2019supermartingale}, \citeyear{bect2019supermartingale}). This criterion is looks for an $\boldsymbol{x} \in \mathcal{X}$, which minimizes:
\begin{equation}
\label{ibv}
    J_n(\boldsymbol{x}) = 
    \mathbb{E}_{n, \boldsymbol{x}} \left[ \int_{\mathcal{X}} p_{n+1}(\boldsymbol{x'})(1-p_{n+1}(\boldsymbol{x'})) \mathbb{P}_{\mathcal{X}}(\mathrm{d}x') \right],
\end{equation}
with $p_{n+1}(\boldsymbol{x'})$ denoting the excursion probability (Eq.~\ref{exprob}) when $\boldsymbol{x_{n+1}}=\boldsymbol{x}$ has been added. Similarly, as for the threshold-weighted CRPS (Section \ref{twCRPS_acquisition_int}), this functional depends on the value in location $\boldsymbol{x}$. Fortunately, for a predictive GP, the IBV criterion can also be simplified with a closed-form expression for the integrand, see \citet{chevalier2014fast}. 

\section{Results}
\label{Sec5}
In this section, we evaluate the performance of the new acquisition strategies and compare it against state-of-the-art methods and a random approach, where the additional training molecules are selected without an acquisition function. To make reading easier amidst the collection of threshold-weighted CRPS terms, we recall that $\evaI$ and $\evaP$ are used to denote the evaluation metrics (Section \ref{twcrps_section}), $\seqIp$ and $\seqPp$ refer to the pointwise acquisition criteria (Section \ref{twCRPS_acquisition_pw}) and $\seqIs$ and $\seqPs$ to the SUR criteria (Section~\ref{twCRPS_acquisition_int}). 

\subsection{Synthetic dataset}

We start with the synthetic test case. Figure \ref{fig:seq_photoS} shows the evolution of the evaluation metrics with $n=30$ initial molecules and adding $\widetilde{n}=50$ more molecules. We depict the mean over 100 repetitions, which are all started with another random initial training set and based on a different validation set. We observe that on the excursion set–specific evaluation metrics (a-f), the acquisition criteria consistently outperform the random approach (with few exceptions). In terms of (a) $\evaI$, both corresponding $\seqIp$ and $\seqIs$ and perform equally well. For (d) $\evaP$ ($\sigmaY=33$), the SUR criterion $\seqPs$ performs noticeably better that pointwise $\seqPp$. Furthermore, after 20 steps, the pointwise criterion $\seqIp$ performs best with respect to (b) $\mathrm{RMSE}_p$ and (e) precision, but worst for (f) sensitivity and (d) $\evaP$. Regarding (e) $\mathrm{RMSE}_{\Gamma}$, $\seqPs$ performs best, while for (f) sensitivity, $\seqPp$ and $\seqPs$ achieve the highest values. Finally, $\seqPs$ performs best with respect to both `basic' (g) $\evaB$ and (h) RMSE. \\

\begin{figure}
\centering\includegraphics[width=\linewidth]{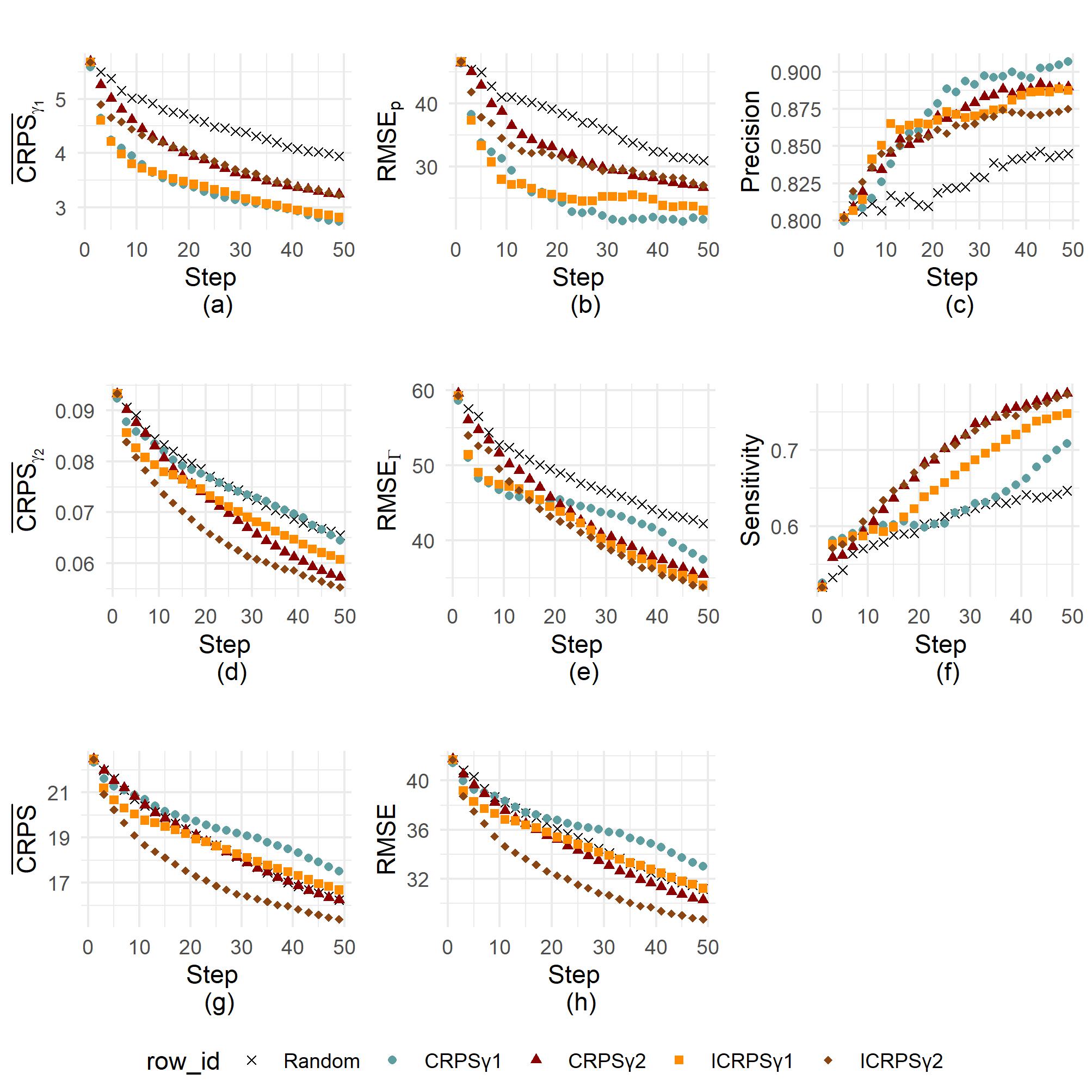} 
  \caption{The sequential performance on the synthetic dataset with $M=100$ validation molecules, $n=30$ and $\widetilde{n} = 50$; the mean over 100 repetitions is depicted and for better visual clarity, only every second step is shown. }
  \label{fig:seq_photoS}
\end{figure}

To better understand the process, Figure \ref{fig:syn_illu} shows an example of an (a) initial model and the evolving performance for $\widetilde{n}=25$ using (b) $\seqIp$ (c) and $\seqPp$. The filled triangles represent the original (noisy) training set, the crosses indicate the predicted mean for the $M=100$ validation molecules, the filled boxes indicate the added molecules, and the dots in the background show the ground truth $f$. It can be clearly seen that while $\seqIp$ excels in accurately predicting the responses of the TP, it is less sensitive and still has more FN compared to $\seqPp$. We observe that $\seqIp$ primarily adds molecules from the excursion set, while $\seqPp$ focuses on molecules near the threshold, as expected based on the two weighting measures.\\

\begin{figure}[p]
  \centering\includegraphics[width=\linewidth]{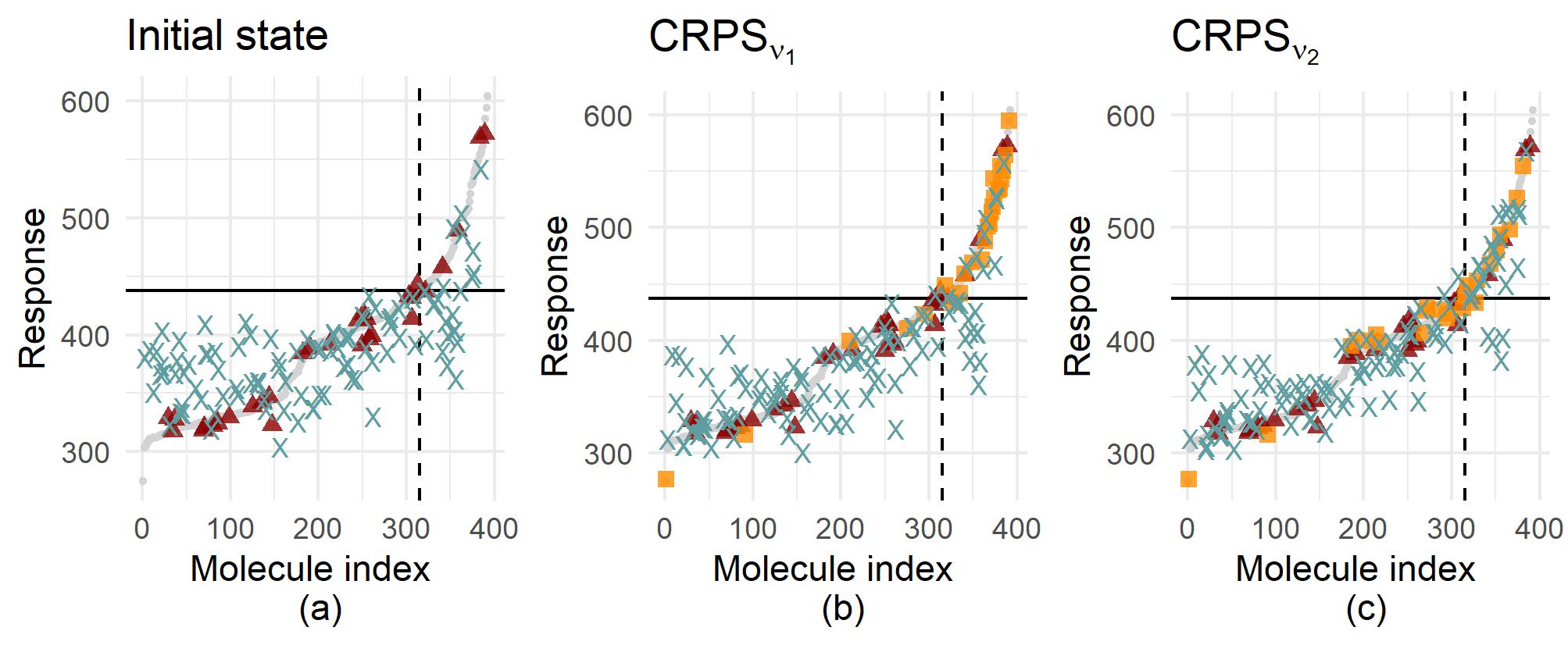}
  \caption{One illustrative example with the synthetic dataset, starting from (a) $n=30$ and adding $\widetilde{n}=25$ molecules with (b) $\seqIp$ and (c) $\seqPp$. The filled triangles show the original (noisy) training set, the crosses the predictive mean for the $M=100$ validation molecules, the filled boxes the added molecules and the dots in the back ground the ground truth $f$. }
  \label{fig:syn_illu}
\end{figure}

Figure \ref{fig:seq_photo_sigS} shows the performance of $\seqPp$ and $\seqPs$ as a function of $\sigmaY$. For both criteria, the traditional CRPS in Figure \ref{fig:seq_photo_sigS}(a) decreases as expected with increasing $\sigmaY$ (because then the criteria focus more and more on the whole domain, as the CRPS does). Both criteria show a clear deterioration in (c) precision with increasing $\sigmaY$. For (b) sensitivity, there is a slight improvement for the pointwise $\seqPp$ with increasing $\sigmaY$, but no systematic effect for $\seqPs$. In all previous and following results we use $\sigmaY=33$, which is half the standard deviation of the responses of the whole dataset. \\

\begin{figure}[p]
\centering
\includegraphics[width=\linewidth]{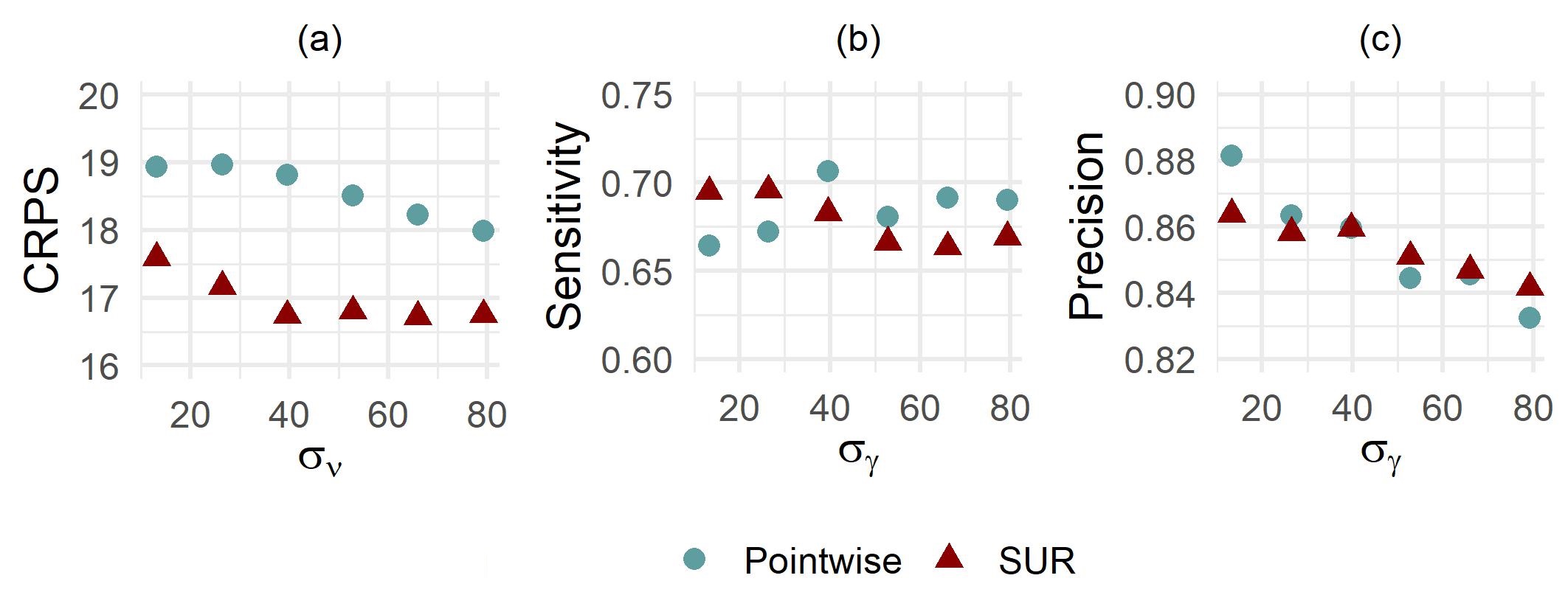} 
  \caption{Performance of $\seqPp$ and $\seqPs$ using different $\sigmaY$ for the synthetic dataset; $M=100$ validation molecules, mean over 50 repetitions, $n=30$ and $\widetilde{n} = 25$. }
  \label{fig:seq_photo_sigS}
\end{figure}

In Figure \ref{fig:syn_comp}, we compare the performance of the new acquisition criteria against the state-of-the-art methods. The box and whiskers plots display the distribution of the metrics obtained in 100 repetitions, highlighting the range, median and quartiles. Also in comparison with the alternative methods, the pointwise $\seqIp$ performs good with respect to the (a) $\evaI$ (reduction of the median compared to the random strategy by about 30~\%), (b) $\mathrm{RMSE}_p$ (minus 40 \% compared to random) and (c) precision (10 \% higher than random). However, $\seqIs$ leads to a similar performance with respect to the median in $\evaI$ and the entropy criteria performs slightly better for precision. For (d) $\evaP$, the good performance of $\seqPs$ is comparable to that of TIMSE, which makes sense since TIMSE also relies on a Gaussian weight around the threshold. Both reduce the median value of the random strategy by about 15 \%. Regarding (e) $\mathrm{RMSE}_{\Gamma}$, the best median value is observed for TMSE (minus 16 \%), closely followed by $\seqPp$, $\seqPs$ and TIMSE. The median of the (f) sensitivity is the highest for $\seqPp$ and $\seqPs$ (15 \% higher than for random), followed by TIMSE and IBV. For reasons of completeness, we also show the results in `basic' (g) $\evaB$ and (h) RMSE. For both metrics, the medians of $\seqPs$ and TIMSE are the lowest, reducing the value of the random strategy by 10 \%. \\

\begin{figure}[p]
  \centering
  \includegraphics[width=\linewidth]{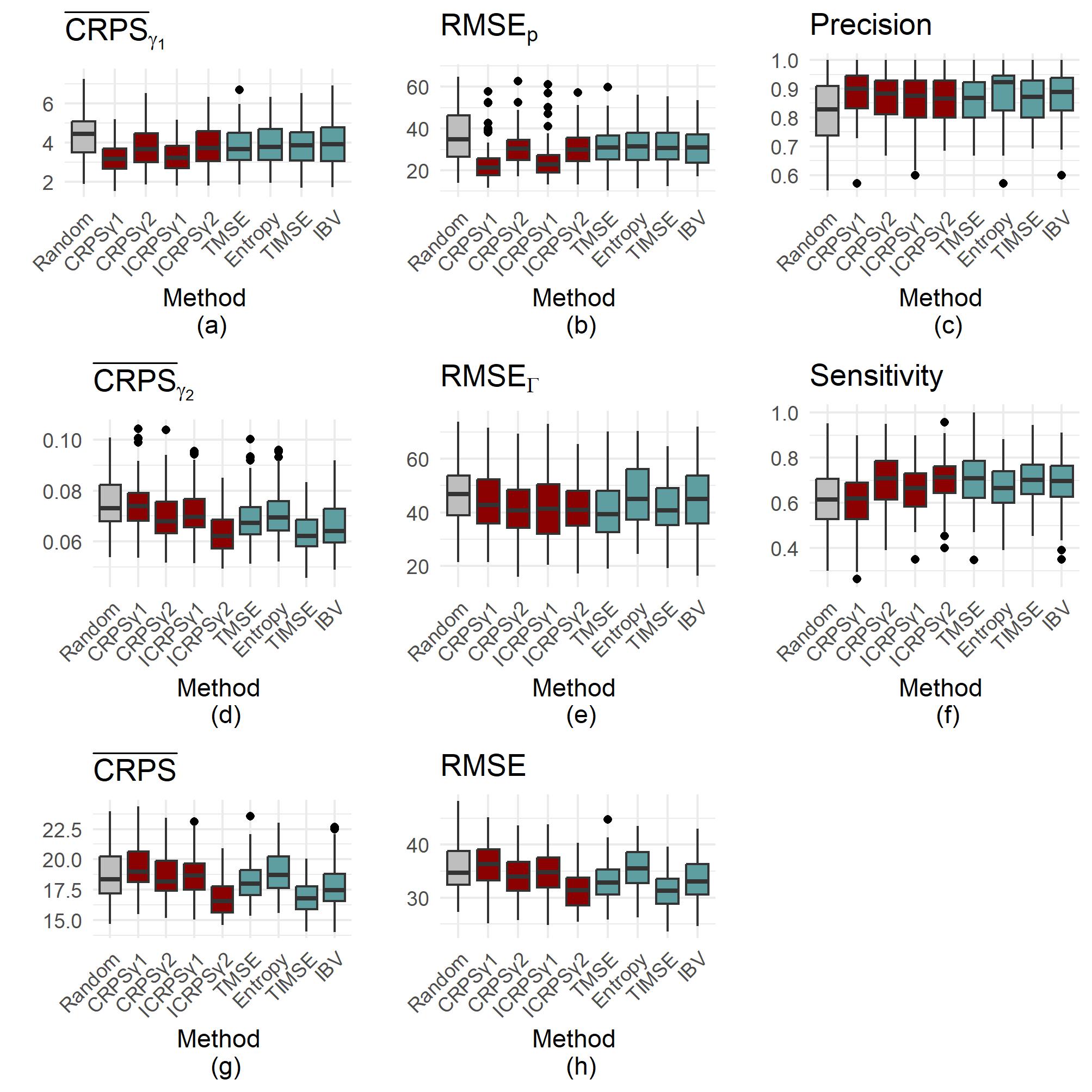}
  \caption{Performance comparison between $\seqIp$, $\seqPp$, $\seqIs$ and $\seqPs$ against alternative selection criteria using the synthetic dataset and $M=100$ validation molecules (100 repetitions, $n=30$ and $\widetilde{n} = 25$). The box and whiskers plots depict the median, quartiles, two whiskers and all 'outlying' points. The whiskers extend from the quartiles to the smallest / highest value at most 1.5 $\times$ the interquartile range away from the quartile, beyond which points are considered outliers. }
  \label{fig:syn_comp}
\end{figure}

\subsection{Original dataset} 

For the original dataset, we only consider the box and whiskers plot for comparison (Fig.~\ref{fig:real_comp}), with focusing solely on (a) $\evaI$, (b) $\evaP$ and (c) $\evaB$ (as the ground truth is not available, see Section \ref{evaluation_metrics}). The pointwise $\seqIp$ also performs best for $\evaI$ in the original dataset and reduces the median value of the random strategy by 30 \%. For $\evaP$, as in the synthetic case, $\seqPs$ and TIMSE lead to similar results, reducing the value of the random strategy by 15 \%. As in the synthetic test case, the lowest value for $\evaB$ is also obtained by $\seqPs$ and TIMSE. \\

\begin{figure}[h]
  \centering
  \includegraphics{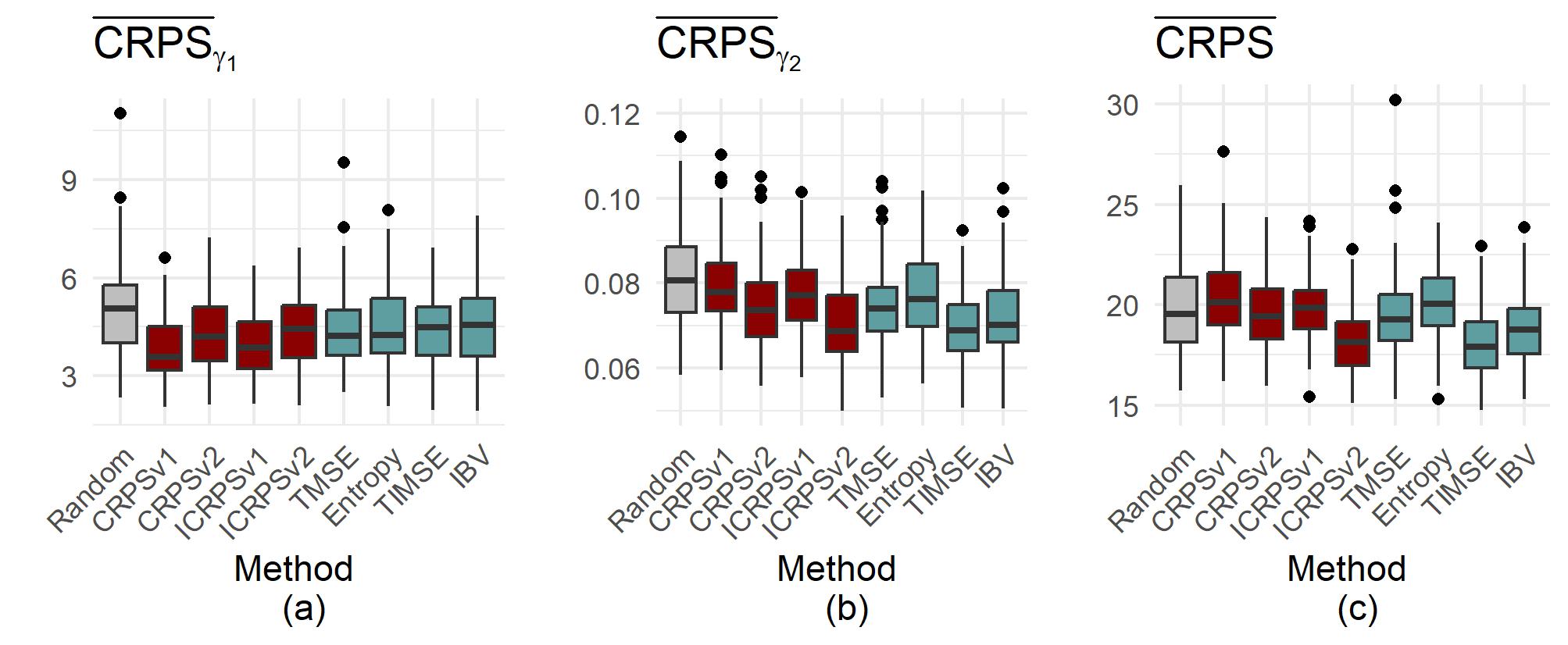}
  \caption{Comparison of the data acquisition methods using the original dataset and $M=100$ validation molecules, 100 repetitions, $n=30$ and $\widetilde{n} = 25$. The box and whiskers plots depict the median, quartiles, two whiskers and all 'outlying' points. The whiskers extend from the quartiles to the smallest / highest value at most 1.5 $\times$ the interquartile range away from the quartile, beyond which points are considered outliers. }
  \label{fig:real_comp}
\end{figure}

\section{Discussion and Conclusions}
\label{Sec6}

In this study, we use a Gaussian Process (GP) approach to model a target function, with a particular focus on an excursion set. We employ targeted sequential design strategies to efficiently select and extend the training dataset, addressing challenges in scenarios where data collection is resource intensive or costly. We introduce novel pointwise and SUR selection criteria based on the threshold-weighted CRPS, using the expected score values. The new criteria show promising performance in both a synthetic and real molecular test case, outperforming or competing with state-of-the-art methods (Fig. \ref{fig:syn_comp} and \ref{fig:real_comp}).  \\

The test case using the synthetic version of the Photoswitch dataset highlights the following strengths and weaknesses of the new selection criteria: In terms of classification, $\seqIp$ is the most precise, while $\seqPp$ and $\seqPs$ are more sensitive (Fig. \ref{fig:seq_photoS} c and f). 
Regarding the prediction for the points identified as part of the excursion set, $\seqIp$ is highly accurate, both when considering the mean alone ($\mathrm{RMSE}_p$, Fig. \ref{fig:seq_photoS}b) and the predictive distribution ($\evaI$, Fig. \ref{fig:seq_photoS}a). However, $\seqIs$ shows a comparable performance and is at the same time more sensitive. Considering the points actually lying in the excursion set, $\seqPs$ shows overall the strongest performance ($\evaP$,$\mathrm{RMSE}_{\Gamma}$ and sensitivity, Fig.~\ref{fig:seq_photoS}d-f). \\

When comparing the new criteria to state-of-the-art methods (Fig. \ref{fig:syn_comp}), we observe that the performance of $\seqPp$ and $\seqPs$ is similar to the existing TMSE and TIMSE criteria, even in terms of $\evaP$. This is expected, as all these methods rely on a Gaussian PDF around the threshold for targeted acquisition. On the other hand, $\seqIp$ and $\seqIs$ clearly outperform the other criteria in terms of $\evaI$ and $\mathrm{RMSE}_p$ (Fig.~\ref{fig:seq_photoS}a+b). 
In the real test case, where no ground truth was available (Fig. \ref{fig:real_comp}), we relied on $\evaI$, $\evaP$ and $\evaB$, observing results similar to those in the synthetic test case, which helped confirm the conclusions. \\

There are several directions in which the presented work can be extended and deepened. Other acquisition criteria, related to expected scoring rules, can also be explored. For instance, the entropy-based criterion \citep{cole2023entropy} is equivalent to the expected log score for a binary predictor. Similarly, the integrand of the IBV is the Bernoulli variance, which corresponds to the expectation of the Brier Score \citep{brier1950verification, murphy1973new}. Future work could explore further alternatives to the CRPS, including multivariate scoring rules \citep[e.g.,][]{allen2023evaluating}. Within the threshold-weighted CRPS, alternative weighting measures could be explored, such as the corresponding CDF of the Gaussian $\wP$ \citep{allen2024}. This would result in a criterion that is less similar to TMSE and TIMSE. As for the criteria based on $\seqIp$, we could adjust the threshold in the weighting measure to incorporate the region around the threshold, thereby enhancing the method's sensitivity. Next, k-step look-ahead strategies could be explored as an alternative to myopic ones. Furthermore, the initial design is currently random, what could be leveraged by employing a space-filling approach.



\section*{Supplementary Material}
\begin{description}
\item PDF containing (1) more on Gaussian Process modeling, (2) a kernel comparison with the (original) Photoswitch dataset, (3) proofs of the analytical expressions for the threshold-weighted CRPS and (4) a proof of consistency for the SUR strategy.
\end{description}

\section*{Acknowledgment of AI Assistance}
In the preparation of this manuscript, ChatGPT was used to check for grammar, clarity, and typographical errors. The AI tool was not used for content generation, data analysis, or substantive writing. All final edits and decisions were made by the authors.

\bibliographystyle{apalike} 
\bibliography{ref}

\end{document}


\def\spacingset#1{\renewcommand{\baselinestretch}%
{#1}\small\normalsize} \spacingset{1.5}

\if0\blind
{\title{\bf CRPS-Based Targeted Sequential Design with Application in Chemical Space \\[10pt] \large{Supplementary Material }}
  \author{Lea Friedli, Athénaïs Gautier, Anna Broccard and 
    David Ginsbourger }
    \date{}
  \maketitle \
} \fi

\appendix

\vspace{-1.5cm}
\section{More on Gaussian Process modeling}
\label{appGPmodeling}

Here, we expand on Section 2.1 and provide additional details on the distribution of a GP~$\boldsymbol{\xi}$ conditioned on \(\mathcal{A}_n\), denoted by \(\boldsymbol{\xi_n} = \left(\xi_n(\boldsymbol{x})\right)_{\boldsymbol{x} \in \mathcal{X}}\). Under mild conditions (invertibility of the matrix \(\boldsymbol{K_{\mathrm{obs}}}\)), \(\boldsymbol{\xi_n}\) remains a GP. Indeed, if the mean $m$ of the prior GP $\boldsymbol{\xi}$ is assumed to be known and constant, this refers to the setting of simple kriging \citep{chiles2012geostatistics}. In practice, the mean function may be unknown, which motivates so-called ordinary kriging with a constant unknown mean and universal kriging, where the mean is assumed to be a linear combination of prescribed basis functions with unknown coefficients. In this work, we focus on the case of ordinary kriging, where $\boldsymbol{\xi_n} \sim \mathrm{GP}\left(m_n, k_n \right)$ with the conditional mean function $m_n : \mathcal{X} \to \mathbb{R}$ and conditional covariance function $k_n : \mathcal{X} \times \mathcal{X} \to \mathbb{R}$ being defined as (for $\boldsymbol{x}, \boldsymbol{x'} \in \mathcal{X}$):
    \begin{align*}
        \label{krig}
        & m_n(\boldsymbol{x}) = \hat{\beta}  + \boldsymbol{K}^\top_{\mathrm{cross}}( \boldsymbol{x} )\boldsymbol{K}_{\mathrm{obs}}^{-1}(\boldsymbol{z_{n}} - \hat{\beta} \mathbf{1}_n ),\\
        &k_n(\boldsymbol{x}, \boldsymbol{x'}) = k(\boldsymbol{x},\boldsymbol{x'}) - \boldsymbol{K}^\top_{\mathrm{cross}}( \boldsymbol{x} ) \boldsymbol{K}_{\mathrm{obs}}^{-1} \boldsymbol{K}_{\mathrm{cross}}( \boldsymbol{x'} ) \nonumber\\
        & \hspace{52pt} + \frac{\left( 1- \boldsymbol{K}^\top_{\mathrm{cross}}( \boldsymbol{x} )  \boldsymbol{K}_{\mathrm{obs}}^{-1} \mathbf{1}_n \right)^\top 
        \left(1- \boldsymbol{K}^\top_{\mathrm{cross}}( \boldsymbol{x'} )  \boldsymbol{K}_{\mathrm{obs}}^{-1} \mathbf{1}_n \right)}{ \mathbf{1}^\top_n  \boldsymbol{K}_{\mathrm{obs}}^{-1} \mathbf{1}_n}, \nonumber
    \end{align*}
    with $\mathbf{1}_n$ denoting the vector of $1$s of size $n$, $\boldsymbol{K}_{\mathrm{cross}}(\boldsymbol{x})=[k(\boldsymbol{x},\boldsymbol{x_i})]_{1 \leq i \leq n}$ and the estimator of the mean $\hat \beta = \frac{\mathbf{1}^\top_n  \boldsymbol{K}_{\mathrm{obs}}^{-1} \boldsymbol{z_{n}} }{\mathbf{1}^\top_n  \boldsymbol{K}_{\mathrm{obs}}^{-1} \mathbf{1}_n}$. \\  
    
In the context of sequential data acquisition, we can facilitate the computations by considering updating formulas \citep{emery2009kriging, chevalier2013corrected}. Considering the expanded observation event $\mathcal{A}_{n+1} = \{ \boldsymbol{Z_n} = \boldsymbol{z_n}, Z_{n+1} = z_{n+1}  \}$ with $z_{n+1} = f(\boldsymbol{x_{n+1}}) + \epsilon_{n+1}$, the mean and covariance function of $\boldsymbol{\xi_{n+1}} \sim \mathrm{GP}(m_{n+1}, k_{n+1})$ can be updated as:  
        \begin{align*}
            &m_{n+1}(\boldsymbol{x}) = m_{n}(\boldsymbol{x}) + \frac{k_{n}(\boldsymbol{x_{n+1}}, \boldsymbol{x})}{(k_{n}(\boldsymbol{x_{n+1}}, \boldsymbol{x_{n+1}}) + \tau^2)} \left(z_{n+1} - m_{n}(\boldsymbol{x_{n+1}}) \right), \\
            &k_{n+1}(\boldsymbol{x}, \boldsymbol{x'}) = k_{n}(\boldsymbol{x}, \boldsymbol{x'}) - \frac{k_{n}(\boldsymbol{x_{n+1}}, \boldsymbol{x})k_{n}(\boldsymbol{x_{n+1}}, \boldsymbol{x'})}{(k_{n}(\boldsymbol{x_{n+1}}, \boldsymbol{x_{n+1}}) + \tau^2)}. \nonumber
        \end{align*}

\section{Kernel comparison with the (original) Photoswitch dataset}
\label{appA}

Classically used families of covariance kernels for Euclidean input spaces include notably the family of isotropic Matérn kernels, encompassing the isotropic exponential and Gaussian (or \text{square-exponential}) covariance kernels as limiting cases \citep{stein2012interpolation}. Isotropy means in this context that kernel values between pairs of locations depend on the Euclidean distance between them. In multidimensional settings, anisotropic extensions of such kernels are also commonplace. The isotropic exponential and Gaussian kernels are given by,
\begin{equation}
    k_{\mathrm{EXP}}(\boldsymbol{x}, \boldsymbol{x'}) = \sigma_k^2 \exp \left( - \frac{||\boldsymbol{x}-\boldsymbol{x'}||}{\theta_k} \right), \quad 
    k_{\mathrm{GAU}}(\boldsymbol{x}, \boldsymbol{x'}) = \sigma_k^2 \exp \left(- \frac{||\boldsymbol{x}-\boldsymbol{x'}||^2}{2 \theta_k^2} \right),
\end{equation}
using the Euclidean distance between the input locations, $||\boldsymbol{x}-\boldsymbol{x'}|| = \sqrt{(\boldsymbol{x}-\boldsymbol{x'})^\top(\boldsymbol{x}-\boldsymbol{x'})}$ and the variance parameter $\sigma_k^2 > 0$ and correlation length $\theta_k > 0$. \\

For the choice of kernel used on the original dataset, we compare the performance of the Tanimoto, Gaussian and Exponential kernels using $10,20,30\%$ of the points for training, as shown in Table \ref{tab:kernel}. The Tanimoto kernel consistently outperforms the Exponential kernel in all scenarios for RMSE and CRPS. Interestingly, while the Gaussian kernel showed comparable performance to the Tanimoto kernel when trained on 20 or 30 $\%$ of the data, it performed exceptionally poorly with only 10$\%$. This may be due to difficulties in estimating the hyperparameters of the kernel with only $n=39$ molecules. Given the consistent performance of the Tanimoto kernel, even with small training datasets, we use it for our sequential work on molecules, especially as it is more interpretable than a stationary kernel in this context. \\

\begin{table}[h]
    \centering
    \begin{tabular}{l|lll|lll|lll}
        Kernel & $k_{\mathrm{TAN}}$ & $k_{\mathrm{EXP}}$ & $k_{\mathrm{GAU}}$  & $k_{\mathrm{TAN}}$ & $k_{\mathrm{EXP}}$ & $k_{\mathrm{GAU}}$ &
         $k_{\mathrm{TAN}}$ & $k_{\mathrm{EXP}}$ & $k_{\mathrm{GAU}}$\\
        $n$ & $10 \%$ &&& $20 \%$ &&& $30 \%$ \\
        \hline
        RMSE & 40.07 & 41.37 & 55.65 & 33.32 & 34.62 & 35.44 & 29.70 & 30.94 & 29.90 \\ 
        CRPS & 21.65 & 22.60 & 31.52 & 17.78 & 18.87 & 18.96 & 15.74 & 16.83 & 15.62 \\ 
    \end{tabular}
    \caption{RMSE and CRPS values for different train/test splits using the Tanimoto, exponential, and Gaussian kernels (averaged over the same 30 random splits for each).}
    \label{tab:kernel}
\end{table}

\section{Proofs threshold-weighted CRPS}

\subsection{Theorem 1: Scoring Rules}
\label{appendix_twcrps1}

It holds that,
\begin{equation}
\seqIp(\GPpredictCDF, \GPobs) = \int_{-\infty}^{\infty} \left[ \Phi \left( \frac{u - \GPpredictmean}{\GPpredictsd} \right) - \mathbbm{1}\{\GPobs \leq u\} \right]^2 \mathbbm{1}\{\t \leq u\} \, \mathrm{d}u = \mathrm{CRPS}(\GPpredictCDF^{\t}, \GPobs^{\t}), \nonumber
\end{equation}
with the traditional CRPS for the left-censored Gaussian $\GPpredictCDF^{\t}(u) = \Phi \left( \frac{u - \GPpredictmean}{\GPpredictsd} \right) \mathbbm{1} \{ u \geq \t \}$ and censored value $\GPobs^{\t} = \max(\GPobs,\t)$ \citep{allen2023weighted}:
\begin{align*}
    \mathrm{CRPS}(\GPpredictCDF^{\t}, \GPobs^{\t}) &= \int_{-\infty}^{\infty} \left[ \Phi \left( \frac{u - \GPpredictmean}{\GPpredictsd} \right) \mathbbm{1} \{ \t \leq u \} - \mathbbm{1}\{\max(\GPobs,\t) \leq u\} \right]^2 \, \mathrm{d}u \\
    &= \int_{-\infty}^{\t} \mathbbm{1}\{\max(\GPobs,\t) \leq u\} \, \mathrm{d}u + \int_{\t}^{\infty} \left[ \Phi \left( \frac{u - \GPpredictmean}{\GPpredictsd} \right) - \mathbbm{1}\{\max(\GPobs,\t) \leq u\} \right]^2 \, \mathrm{d}u \\
    &= \int_{\infty}^{\infty} \left[ \Phi \left( \frac{u - \GPpredictmean}{\GPpredictsd} \right) - \mathbbm{1}\{\max(\GPobs, \t) \leq u\} \right]^2 \mathbbm{1} \{ u \geq \t \} \, \mathrm{d}u \\
    &= \int_{\infty}^{\infty} \left[ \Phi \left( \frac{u - \GPpredictmean}{\GPpredictsd} \right) - \mathbbm{1}\{\GPobs \leq u\} \mathbbm{1}\{\t \leq u\} \right]^2 \mathbbm{1} \{ u \geq \t \} \, \mathrm{d}u \\
    &= \int_{-\infty}^{\infty} \left[ \Phi \left( \frac{u - \GPpredictmean}{\GPpredictsd} \right) - \mathbbm{1}\{\GPobs \leq u\} \right]^2 \mathbbm{1}\{u \geq \t \} \, \mathrm{d}u.
\end{align*}

The CRPS for a censored normal distribution has an analytical formula \citep{jordan1019}. For a left-censored standard normal distribution $\Phi^{\t}$ with upper bound $\infty$, lower bound $\t$ and a point mass of $L$ at the lower bound, it holds,
\begin{align*}
    \mathrm{CRPS}(\Phi^{\t}, y^{\t}) &= -t L^2
    + \frac{(1-L)}{\left( 1 - \Phi(\t) \right) } y^{\t} \left( 2 \Phi(y^{\t}) - \frac{(1-2L) + \Phi(\t)}{1-L}  \right) \\ 
    &\quad \quad + \frac{(1-L)}{\left( 1 - \Phi(\t) \right) } \left( 2 \phi \left( y^{\t} \right) - 2 \phi \left( \t \right) L \right)\\ 
    &\quad \quad - \left(\frac{(1-L)}{\left( 1 - \Phi(\t) \right) }\right)^2  \frac{1}{\sqrt{\pi}} \left( 1 - \Phi(\t \sqrt{2}) \right).
\end{align*}

By using the transformation formula for a censored normal distribution $\GPpredictCDF^{\t}$ \citep{jordan1019} and the point mass in $\t$ being $L = \GPpredictCDF(\t) = \Phi(\widetilde{\t})$ for $\widetilde{\t} = (\t - \GPpredictmean)/\GPpredictsd$, we denote by $\widetilde{\GPobs^t} = \Big(\max(\GPobs,\t) - \GPpredictmean \Big)/\GPpredictsd$ and get,
\begin{align*}
    \seqIp(\GPpredictCDF, \GPobs) &=\mathrm{CRPS}(\GPpredictCDF^{\t}, \GPobs^{\t}) 
    = \GPpredictsd \; \mathrm{CRPS}\left( \Phi^{\widetilde{\t}}, \widetilde{y^{\t}} \right) \\
    &= \GPpredictsd \left[ -\widetilde{\t} \Phi(\widetilde{\t})^2
    + \frac{\left(1-\Phi(\widetilde{\t})\right)}{\left( 1 - \Phi(\widetilde{\t}) \right) } \widetilde{y^{\t}} \left( 2 \Phi(\widetilde{y^{\t}}) - \frac{\left(1-2\Phi(\widetilde{\t})\right) + \Phi(\widetilde{\t})}{1-\Phi(\widetilde{\t})}  \right) \right.\\ 
    &\quad \quad + \frac{\left(1-\Phi(\widetilde{\t})\right)}{\left( 1 - \Phi(\widetilde{\t}) \right) } \left( 2 \phi \left( \widetilde{y^{\t}} \right) - 2 \phi \left( \widetilde{\t} \right) \Phi(\widetilde{\t}) \right) \left.- \left(\frac{\left(1- \Phi(\widetilde{\t})\right)}{\left( 1 - \Phi(\widetilde{\t}) \right) }\right)^2  \frac{1}{\sqrt{\pi}} \left( 1 - \Phi(\widetilde{\t} \sqrt{2}) \right) \right] \\
    &= \GPpredictsd \left[ -\widetilde{\t} \Phi(\widetilde{\t})^2
    + \widetilde{y^{\t}} \left( 2 \Phi(\widetilde{y^{\t}}) - 1  \right) +  \left( 2 \phi \left( \widetilde{y^{\t}} \right) - 2 \phi \left( \widetilde{\t} \right) \Phi(\widetilde{\t}) \right) -  \frac{1}{\sqrt{\pi}} \left( 1 - \Phi(\widetilde{\t} \sqrt{2}) \right) \right].
\end{align*}

For $\wP$, we derive:
\begin{align*}
    \seqPp(\GPpredictCDF, \GPobs) 
    &= \int_{-\infty}^{\infty} \left[ \Phi \left( \frac{u-\GPpredictmean}{\GPpredictsd} \right) - \mathbbm{1}\{\GPobs \leq u\} \right]^2 \phi \left( \frac{u -\t}{\sigmaY} \right) \frac{1}{\sigmaY} \, \mathrm{d}u \\
    &= \underbrace{\int_{-\infty}^{\infty}  \Phi \left( \frac{u-\GPpredictmean}{\GPpredictsd} \right)^2 \phi \left( \frac{u -\t}{\sigmaY} \right) \frac{1}{\sigmaY} \, \mathrm{d}u}_{(1)} \\
    &\quad \quad -2 \underbrace{\int_{-\infty}^{\infty} \Phi \left( \frac{u-\GPpredictmean}{\GPpredictsd} \right) \mathbbm{1}\{\GPobs \leq u\} \phi \left( \frac{u -\t}{\sigmaY} \right) \frac{1}{\sigmaY} \, \mathrm{d}u}_{(2)} \\
    &\quad \quad + \underbrace{\int_{-\infty}^{\infty} \mathbbm{1}\{\GPobs \leq u\} \phi \left( \frac{u -\t}{\sigmaY} \right) \frac{1}{\sigmaY} \, \mathrm{d}u}_{(3)}.
\end{align*}

The three parts can be simplified as follows: \\
(1) For $N, \widetilde{N}, U' \overset{i.i.d.}{\sim} \mathcal{N}(0,1)$ and using a variable transform $u' = \frac{u -\t}{\sigmaY}$:
\begin{align*}
    \int_{-\infty}^{\infty}  \Phi \left( \frac{u-\GPpredictmean}{\GPpredictsd} \right)^2 \frac{1}{\sigmaY} \phi \left( \frac{u -\t}{\sigmaY} \right)  \, \mathrm{d}u
    &= \int_{-\infty}^{\infty}  \Phi \left( \frac{\sigmaY u' + \t -\GPpredictmean}{\GPpredictsd} \right)^2 \phi \left( u' \right)  \, \mathrm{d}u' \\
    &= \int_{-\infty}^{\infty}  \mathbb{P} \left( N \leq \frac{\sigmaY u' + \t -\GPpredictmean}{\GPpredictsd} \right) \mathbb{P} \left( \widetilde{N} \leq \frac{\sigmaY u' + \t -\GPpredictmean}{\GPpredictsd} \right) \phi \left( u' \right)  \, \mathrm{d}u'\\
    &=  \mathbb{P} \left( \GPpredictsd N -\t + \GPpredictmean - \sigmaY U' \leq 0, \GPpredictsd \widetilde{N} -\t + \GPpredictmean - \sigmaY U' \leq 0  \right) \\
    &= \Phi \left( \begin{pmatrix} 0 \\ 0 \end{pmatrix}; \begin{pmatrix} \GPpredictmean - \t \\ \GPpredictmean - \t \end{pmatrix},  
    \begin{pmatrix} \sigmaY^2 + \GPpredictvar & \sigmaY^2 \\ \sigmaY^2 & \sigmaY^2 + \GPpredictvar \end{pmatrix}  \right)
\end{align*}

(2) For $V, U'  \overset{i.i.d.}{\sim} \mathcal{N}(0,1)$ and using again $u' = \frac{u -\t}{\sigmaY}$:
\begin{align*}
    \int_{-\infty}^{\infty} &\Phi \left( \frac{u-\GPpredictmean}{\GPpredictsd} \right) \mathbbm{1}\{\GPobs \leq u\} \phi \left( \frac{u -\t}{\sigmaY} \right) \frac{1}{\sigmaY} \, \mathrm{d}u \\
    &= \int_{-\infty}^{\infty} \Phi \left( \frac{\sigmaY u' + \t -\GPpredictmean}{\GPpredictsd} \right) \mathbbm{1}\{\GPobs \leq \sigmaY u' + \t \} \; \phi \left( u' \right) \, \mathrm{d}u' \\
    &= \int_{-\infty}^{\infty} \int_{-\infty}^{\infty} 
    \mathbbm{1} \{ v \leq \frac{\sigmaY u' + \t -\GPpredictmean}{\GPpredictsd} \} \, \phi(v) \, \mathrm{d}v \,
     \mathbbm{1}\{\GPobs \leq \sigmaY u' + \t \} \; \phi \left( u' \right) \, \mathrm{d}u' \\
    &=  \mathbb{P} \left( \GPpredictsd V - 
    \sigmaY U' + \GPpredictmean  - \t \leq 0, \GPobs - \sigmaY U' - \t \leq 0  \right) \\
    &= \Phi \left( \begin{pmatrix} 0 \\ 0 \end{pmatrix}; \begin{pmatrix} \GPpredictmean - \t \\ y - \t \end{pmatrix},  
    \begin{pmatrix} \sigmaY^2 + \GPpredictvar & \sigmaY^2 \\ \sigmaY^2 & \sigmaY^2  \end{pmatrix}  \right)
\end{align*}

(3) Reusing $u' = \frac{u -\t}{\sigmaY}$, it holds that:
\begin{align*}
   \int_{-\infty}^{\infty} \mathbbm{1}\{\GPobs \leq u\} \phi \left( \frac{u -\t}{\sigmaY} \right) \frac{1}{\sigmaY} \, \mathrm{d}u 
   = \int_{-\infty}^{\infty} \mathbbm{1} \{ \frac{\GPobs -\t}{\sigmaY} \leq u' \}  \phi \left( u' \right)  \mathrm{d}u' 
   = \Phi \left( \frac{\t - \GPobs}{\sigmaY} \right)
\end{align*}

Eventually, we obtain:
\begin{align*}
    \seqPp(\GPpredictCDF, \GPobs) 
    &= \Phi \left( \begin{pmatrix} 0 \\ 0 \end{pmatrix}; \begin{pmatrix} \GPpredictmean - \t \\ \GPpredictmean - \t \end{pmatrix},  
    \begin{pmatrix} \sigmaY^2 + \GPpredictvar & \sigmaY^2 \\ \sigmaY^2 & \sigmaY^2 + \GPpredictvar \end{pmatrix}  \right) \\
    &\quad \quad -2  
    \Phi \left( \begin{pmatrix} 0 \\ 0 \end{pmatrix}; \begin{pmatrix} \GPpredictmean - \t \\ y - \t \end{pmatrix},  
    \begin{pmatrix} \sigmaY^2 + \GPpredictvar & \sigmaY^2 \\ \sigmaY^2 & \sigmaY^2  \end{pmatrix}  \right) + \Phi \left( \frac{\t - \GPobs}{\sigmaY} \right).
\end{align*}

\subsection{Theorem 2: Pointwise criteria}
\label{appendix_twcrps2}

Let us start with the general result, where $\rho(\cdot)$ denotes the PDF $F$:
    \begin{align*}
   \twcrps(F) 
   &= \int_{-\infty}^{\infty} \twcrps (F, y') \rho(y') \, \mathrm{d}y'  \\
   &= \int_{-\infty}^{\infty} \int_{-\infty}^{\infty} [F(u) - \mathbbm{1}\{y' \leq u\}]^2 \, \intB \, \rho(y') \, \mathrm{d}y'  \\
   &= \int_{-\infty}^{\infty} \int_{-\infty}^{\infty} \left[ F(u)^2 - 2 F(u) \mathbbm{1}\{y' \leq u\} + \mathbbm{1}\{y' \leq u\} \right] \, \rho(y') \, \mathrm{d}y' \, \intB \\
   &= \int_{-\infty}^{\infty} \left[
   F(u)^2 - 2 F(u)^2 + F(u) \right] \,  \intB \\
   &= \int F(u) \left( 1-F(u) \right) \, \intB
\end{align*}

Let us consider $F = \GPpredictCDF$ and the first representation for $\wI$. With defining $u^{\left(\widetilde{\t}\right)} = \max(u,\widetilde{t})$ and transforming $u = \frac{\GPobs'-\GPpredictmean}{\GPpredictsd} $, we get:
\begin{align*}
    \seqIp(\GPpredictCDF) &=\int  \seqIp(\GPpredictCDF, \GPobs') \GPpredictPDF(y') \mathrm{d} y' \\
    &= \int \mathrm{CRPS}(\GPpredictCDF^{\t}, \max(\GPobs',\t)) \phi \left( \frac{\GPobs'-\GPpredictmean}{\GPpredictsd} \right) \frac{1}{\GPpredictsd}  \mathrm{d} y' \\
    &= \int \mathrm{CRPS}(\GPpredictCDF^{\t}, \max(\GPpredictsd u + \GPpredictmean, \t) )\phi(u) \mathrm{d} u \\
    &= \int \GPpredictsd \mathrm{CRPS} \left( \Phi^{\widetilde{\t}}, u^{\left(\widetilde{\t}\right)} \right)\phi(u) \mathrm{d} u \\
    &= \int \GPpredictsd \Big[ -\widetilde{\t} \Phi(\widetilde{\t})^2
    + u^{\left(\widetilde{\t}\right)} \left( 2 \Phi\left( u^{\left(\widetilde{\t}\right)} \right) - 1  \right) +  \left( 2 \phi \left( u^{\left(\widetilde{\t}\right)} \right) - 2 \phi \left( \widetilde{\t} \right) \Phi(\widetilde{\t}) \right)  \\
    & \quad \quad - \frac{1}{\sqrt{\pi}} \left( 1 - \Phi(\widetilde{\t} \sqrt{2}) \right) \Big]   \phi(u) \mathrm{d} u \\
    &= \GPpredictsd \underbrace{\int -\widetilde{\t} \Phi(\widetilde{\t})^2 \phi(u) \mathrm{d} u}_{= -\widetilde{\t} \Phi(\widetilde{\t})^2}
    + \GPpredictsd \underbrace{\int u^{\left(\widetilde{\t}\right)} \left( 2 \Phi(u^{\left(\widetilde{\t}\right)}) - 1  \right) \phi(u) \mathrm{d} u }_{(1)}
    \\ & \quad \quad
    + \GPpredictsd \underbrace{\int \left( 2 \phi \left( u^{\left(\widetilde{\t}\right)} \right) - 2 \phi \left( \widetilde{\t} \right) \Phi(\widetilde{\t}) \right) \phi(u) \mathrm{d} u }_{(2)}
    - \GPpredictsd \underbrace{\int \frac{1}{\sqrt{\pi}} \left( 1 - \Phi(\widetilde{\t} \sqrt{2}) \right) \phi(u) \mathrm{d} u}_{= \frac{1}{\sqrt{\pi}} \left( 1 - \Phi(\widetilde{\t} \sqrt{2}) \right)}.
\end{align*}

The details for the two numbered parts are as follows: \\
(1) Using \citet{owen1980table} Equations (1011.1) and (11):
\begin{align*}
    \int u^{\left(\widetilde{\t}\right)} &\left( 2 \Phi \left(u^{\left(\widetilde{\t}\right)} \right) - 1  \right) \phi(u) \mathrm{d} u = 
    \int_{-\infty}^{\widetilde{\t}} \widetilde{\t} \left( 2 \Phi(\widetilde{\t}) - 1  \right) \phi(u) \mathrm{d} u 
    + \int_{\widetilde{\t}}^{\infty} u \left( 2 \Phi(u) - 1  \right) \phi(u) \mathrm{d} u \\
    &= \widetilde{\t} \left( 2 \Phi(\widetilde{\t}) - 1  \right) \Phi(\widetilde{\t}) 
    + 2\int_{\widetilde{\t}}^{\infty} u \Phi(u)  \phi(u) \mathrm{d} u - \int_{\widetilde{\t}}^{\infty} u \phi(u) \mathrm{d} \\
    &= \widetilde{\t} \left( 2 \Phi(\widetilde{\t}) - 1  \right) \Phi(\widetilde{\t}) 
    + 2 \Big[\frac{1}{2 \sqrt{\pi}} \Phi(u\sqrt{2}) - \phi(u) \Phi(u)  \Big]_{\widetilde{\t}}^{\infty}
    - \Big[-\phi(u) \Big]_{\widetilde{\t}}^{\infty} \\
    &= \widetilde{\t} \left( 2 \Phi(\widetilde{\t}) - 1  \right) \Phi(\widetilde{\t}) 
    + \frac{1}{\sqrt{\pi}} - \frac{1}{\sqrt{\pi}} \Phi(\widetilde{\t} \sqrt{2}) + 2 \phi(\widetilde{\t}) \Phi(\widetilde{\t}) - \phi(\widetilde{\t}) \\
    &= 2 \widetilde{\t} \Phi(\widetilde{\t})^2 - \widetilde{\t}\Phi(\widetilde{\t}) 
    + \frac{1}{\sqrt{\pi}} - \frac{1}{\sqrt{\pi}} \Phi(\widetilde{\t} \sqrt{2}) + 2 \phi(\widetilde{\t}) \Phi(\widetilde{\t}) - \phi(\widetilde{\t})
\end{align*}

(2) Using \citet{owen1980table} Equation (20)
\begin{align*}
    \int & \left( 2 \phi \left( u^{\left(\widetilde{\t}\right)} \right) - 2 \phi \left( \widetilde{\t} \right) \Phi(\widetilde{\t}) \right) \phi(u) \mathrm{d} u \\ 
    &= \int_{-\infty}^{\widetilde{\t}} \left( 2 \phi \left( \widetilde{\t} \right) - 2 \phi \left( \widetilde{\t} \right) \Phi(\widetilde{\t}) \right) \phi(u) \mathrm{d} u  
    + \int_{\widetilde{\t}}^{\infty} \left( 2 \phi \left( u \right) - 2 \phi \left( \widetilde{\t} \right) \Phi(\widetilde{\t}) \right) \phi(u) \mathrm{d} u \\
    &= \left( 2 \phi \left( \widetilde{\t} \right) - 2 \phi \left( \widetilde{\t} \right) \Phi(\widetilde{\t}) \right) \Phi(\widetilde{\t})
    + 2 \int_{\widetilde{\t}}^{\infty} \phi \left( u \right)^2 \mathrm{d} u - 2 \phi \left( \widetilde{\t} \right) \Phi(\widetilde{\t}) \int_{\widetilde{\t}}^{\infty} \phi(u)  \mathrm{d} u \\
    &= \left( 2 \phi \left( \widetilde{\t} \right) - 2 \phi \left( \widetilde{\t} \right) \Phi(\widetilde{\t}) \right) \Phi(\widetilde{\t})
    + 2 \Big[\frac{1}{2\sqrt{\pi}} \Phi(u \sqrt{2}) \Big]_{\widetilde{\t}}^{\infty}  - 2 \phi \left( \widetilde{\t} \right) \Phi(\widetilde{\t}) (1-\Phi(\widetilde{\t})) \\
    &= \left( 2 \phi \left( \widetilde{\t} \right) - 2 \phi \left( \widetilde{\t} \right) \Phi(\widetilde{\t}) \right) \Phi(\widetilde{\t})
    + \frac{1}{\sqrt{\pi}} \left(1 -  \Phi(\widetilde{\t} \sqrt{2}) \right)  - 2 \phi \left( \widetilde{\t} \right) \Phi(\widetilde{\t}) (1-\Phi(\widetilde{\t})) \\
    &= 2 \phi \left( \widetilde{\t} \right) \Phi(\widetilde{\t}) - 2 \phi \left( \widetilde{\t} \right) \Phi(\widetilde{\t})^2
    + \frac{1}{\sqrt{\pi}} \left(1 -  \Phi(\widetilde{\t} \sqrt{2}) \right)  - 2 \phi \left( \widetilde{\t} \right) \Phi(\widetilde{\t}) + 2 \phi \left( \widetilde{\t} \right)\Phi(\widetilde{\t})^2 \\
    &= \frac{1}{\sqrt{\pi}} \left(1 -  \Phi(\widetilde{\t} \sqrt{2}) \right)  
\end{align*}

Eventually, we get: 
\begin{align*}
    \seqIp(\GPpredictCDF) &= \GPpredictsd \Big[ -\widetilde{\t} \Phi(\widetilde{\t})^2 +
    2 \widetilde{\t} \Phi(\widetilde{\t})^2 - \widetilde{\t}\Phi(\widetilde{\t})  + \frac{1}{\sqrt{\pi}} - \frac{1}{\sqrt{\pi}} \Phi(\widetilde{\t} \sqrt{2}) + 2 \phi(\widetilde{\t}) \Phi(\widetilde{\t}) - \phi(\widetilde{\t}) \\
    & \quad \quad
    + \frac{1}{\sqrt{\pi}} \left(1 -  \Phi(\widetilde{\t} \sqrt{2}) \right)  - \frac{1}{\sqrt{\pi}} \left( 1 - \Phi(\widetilde{\t} \sqrt{2}) \right)
    \Big] \\
    &= \GPpredictsd \Big[ \widetilde{\t} \Phi(\widetilde{\t})^2 - \widetilde{\t}\Phi(\widetilde{\t})  + \frac{1}{\sqrt{\pi}} - \frac{1}{\sqrt{\pi}} \Phi(\widetilde{\t} \sqrt{2}) + 2 \phi(\widetilde{\t}) \Phi(\widetilde{\t}) - \phi(\widetilde{\t}) \Big]
\end{align*}

The alternative representation is obtained using $N, \widetilde{N}  \overset{i.i.d.}{\sim} \mathcal{N}(0,1)$ and $u' = \frac{u - \GPpredictmean}{\GPpredictsd}$:
\begin{align*}
    \seqIp(\GPpredictCDF) 
    &= \int_{\t}^{\infty}
    \Phi \left( \frac{u - \GPpredictmean}{\GPpredictsd} \right) \left( 1- \Phi \left( \frac{u - \GPpredictmean}{\GPpredictsd} \right) \right)  \, \mathrm{d}u \\
    &= \int_{\t}^{\infty} 
    \mathbb{P} \left( N \leq \frac{u - \GPpredictmean}{\GPpredictsd} \right) 
    \mathbb{P} \left( \widetilde{N} \geq \frac{u - \GPpredictmean}{\GPpredictsd} \right)  \, \mathrm{d}u \\
    &= \int_{\t}^{\infty} \mathbb{E} \left[
    \mathbbm{1} \left( N \leq \frac{u - \GPpredictmean}{\GPpredictsd} \leq \widetilde{N}  \right) \right] \, \mathrm{d}u \\
    &= \mathbb{E} \left[ \int
    \mathbbm{1} \left( \max \left( N,\frac{\t-\GPpredictmean}{\GPpredictsd} \right) \leq \frac{u - \GPpredictmean}{\GPpredictsd} \leq \widetilde{N}  \right)  \, \mathrm{d}u  \right]\\
    &= \mathbb{E} \left[ \GPpredictsd \int
    \mathbbm{1} \left( \max \left( N,\frac{\t-\GPpredictmean}{\GPpredictsd} \right) \leq u' \leq \widetilde{N}  \right)  \, \mathrm{d}u'  \right]\\
    &= \GPpredictsd \mathbb{E} \left[ \left( \widetilde{N} - \max \left( N, \frac{\t-\GPpredictmean}{\GPpredictsd} \right) \right)_+ \right].
\end{align*}

For $\wP$, assuming $N, \widetilde{N}, N' \sim \mathcal{N}(0,1)$ independent and transform $n' = \frac{u - \t}{\sigmaY}$, we get:
\begin{align*}
    \seqPp(\GPpredictCDF) 
    &= \int 
    \Phi \left( \frac{u - \GPpredictmean}{\GPpredictsd} \right) \left( 1- \Phi \left( \frac{u - \GPpredictmean}{\GPpredictsd} \right) \right) \frac{1}{\sigmaY} \phi \left( \frac{u-t}{\sigmaY} \right) \, \mathrm{d}u \\
    &= \int     \mathbb{P} \left( N \leq \frac{u - \GPpredictmean}{\GPpredictsd} \right) \mathbb{P} \left( \widetilde{N} \leq - \frac{u - \GPpredictmean}{\GPpredictsd} \right)
    \frac{1}{\sigmaY} \phi \left( \frac{u-t}{\sigmaY} \right) \, \mathrm{d}u \\
    &= \int     \mathbb{P} \left( N \leq \frac{t + \sigmaY n' - \GPpredictmean}{\GPpredictsd} \right) \mathbb{P} \left( \widetilde{N} \leq - \frac{t + \sigmaY n' - \GPpredictmean}{\GPpredictsd} \right)
     \phi \left( n' \right) \, \mathrm{d}n' \\
    &= \mathbb{P} \left( N' \geq \frac{\GPpredictsd N - \t + \GPpredictmean}{\sigmaY}, 
    N' \leq \frac{- \GPpredictsd \widetilde{N} -\t + \GPpredictmean}{\sigmaY} \right)\\
    &= \mathbb{P} \left( \GPpredictsd N - \t + \GPpredictmean - \sigmaY N' \leq 0, 
    \GPpredictsd \widetilde{N} +\t - \GPpredictmean + \sigmaY  N' \leq 0 \right) \\
    &= \Phi \left( \begin{pmatrix} 0 \\ 0 \end{pmatrix}; \begin{pmatrix}  \GPpredictmean - \t \\ \t -\GPpredictmean \end{pmatrix},  
    \begin{pmatrix} \sigmaY^2 + \GPpredictvar & -\sigmaY^2 \\ -\sigmaY^2 & \sigmaY^2 + \GPpredictvar \end{pmatrix}  \right)  
\end{align*}

\subsection{Theorem 3: SUR criteria}
\label{appendix_twcrps3}

Recall that we consider the conditional GP $\GPupdate \sim \mathrm{GP}(\GPupdateM, \GPupdateK)$ and its predecessor $\boldsymbol{\xi_n} \sim \mathrm{GP}(m_n, k_n)$. For fixed $\boldsymbol{x'} \in \mathcal{X}$ and using $\sigma^2_{n+1} = \GPupdateK(\boldsymbol{x'}, \boldsymbol{x'})$, $\mu_n = m_n(\boldsymbol{x'})$, $\alpha_n = \alpha_{n}(\boldsymbol{x'})$, we start with some general reformulations ($N, \widetilde{N}, V \overset{i.i.d.}{\sim} \mathcal{N}(0,1)$):
\begin{align*}
    \mathbb{E}_{n,\boldsymbol{x}} \left[  \twcrps(\GPupdateU(\boldsymbol{x'}))  \right] 
    &= \int  \twcrps (\Phi_{\mu_n + \alpha_{n} v, \sigma_{n+1}^2}) \phi(v) \; \mathrm{d}v \\
    &= \int \int \Phi \left( \frac{u - \mu_n - \alpha_{n} v}{\sigma_{n+1}} \right) \left( 1- \Phi \left( \frac{u - \mu_n - \alpha_{n} v}{\sigma_{n+1}} \right) \right)  \phi(v) \; \mathrm{d}v  \; \intB  \\ 
    &= \int \int \mathbb{P} \left( \widetilde{N} \leq \frac{u - \mu_n - \alpha_{n} v}{\sigma_{n+1}} \right) \mathbb{P} \left(N \geq \frac{u - \mu_n - \alpha_{n}v}{\sigma_{n+1}} \right) \phi(v) \; \mathrm{d}v \;\intB  \\ 
    &= \int \mathbb{P} \left( \widetilde{N} \leq \frac{u - \mu_n - \alpha_{n} V}{\sigma_{n+1}} \right) \mathbb{P} \left(N \geq \frac{u - \mu_n - \alpha_{n} V}{\sigma_{n+1}} \right) \; \intB  
\end{align*}

Moving to $\wI$, we use $u' = \frac{u - \mu_n - \alpha_{n} V}{\sigma_{n+1}}$ and obtains: 
\begin{align*}
    &\mathbb{E}_{n,\boldsymbol{x}} \left[  \seqIp(\GPupdateU(\boldsymbol{x'}))  \right] \\
    &= \int_{\t}^{\infty} \mathbb{E} \left[
    \mathbbm{1} \{\widetilde{N} \leq \frac{u - \mu_n - \alpha_{n} V}{\sigma_{n+1}}  \leq N \} \right] \;\mathrm{d}u \\
    &= \sigma_{n+1} \mathbb{E} \left[ \int_{\infty}^{\infty}
    \mathbbm{1} \{ \max \left(\widetilde{N}, \frac{\t - \mu_n - \alpha_{n} V}{\sigma_{n+1}} \}   \right) \leq u' \leq N \} \; \mathrm{d}u' \right]  \\
    &= \sigma_{n+1} \mathbb{E} \left[ \left( N - \max \left(\widetilde{N}, \frac{\t - \mu_n - \alpha_{n} V}{\sigma_{n+1}} \}   \right) \right)_+ \right]
\end{align*}

Eventually, for $\wP$ using $N, \widetilde{N}, V, U' \overset{i.i.d.}{\sim} \mathcal{N}(0,1)$ and $u' = \frac{u-\t}{\sigmaY}$:
\begin{align*}
    &\mathbb{E}_{n,\boldsymbol{x}} \left[  \seqPp(\GPupdateU(\boldsymbol{x'}))  \right] 
    = \int \mathbb{P} \left( \sigma_{n+1} \widetilde{N} + \alpha_{n} V + \mu_n \leq u \leq \sigma_{n+1} N + \mu_n + \alpha_{n} V  \right) \;  \phi \left( \frac{u-\t}{\sigmaY} \right) \frac{1}{\sigmaY} \;\mathrm{d}u \\
    &= \int \mathbb{P} \left( \sigma_{n+1} \widetilde{N} + \alpha_{n} V + \mu_n \leq \sigmaY u' + \t \leq \sigma_{n+1} N + \mu_n + \alpha_{n} V  \right) \phi \left( u' \right) \;\mathrm{d}u' \\
    &= \mathbb{P} \left( \sigma_{n+1} \widetilde{N} + \alpha_{n} V + \mu_n - \sigmaY U' - \t \leq 0, 
    \sigmaY U' + \t - \sigma_{n+1} N - \mu_n - \alpha_{n} V \leq 0 \right) \\
    &= \Phi \left( \begin{pmatrix} 0 \\ 0 \end{pmatrix}; \begin{pmatrix}  \mu_n - t \\ \t - \mu_n \end{pmatrix},  
    \begin{pmatrix} 
    \sigma_{n+1}^2 + \alpha_{n}^2 + \sigmaY^2 
    & -\alpha_{n}^2 - \sigmaY^2 \\ 
    -\alpha_{n}^2 - \sigmaY^2 
    & \sigma_{n+1}^2 + \alpha_{n}^2 + \sigmaY^2
    \end{pmatrix}  \right)  
\end{align*}

\hspace{1cm}
\section{Consistency of SUR strategy based on Itw-CRPS}

We establish a consistency result for a SUR strategy relying on an integrated CRPS criterion. The settings and approach build on those of \citet{bect2019supermartingale}. A particular effort was made to stick to the notation of \citet{bect2019supermartingale} to ease referring to its results. In particular, to retain full generality, we work with a reference measure $\mu$ in the forthcoming proof, while other parts of our work focus on the case where $\mu$ is the Lebesgue measure. Our new result and proof turn out to be very close to those established in \citet{bect2019supermartingale} for the case of the integrated Bernoulli variance. Running assumptions of \citet{bect2019supermartingale}, that are recalled thereafter are assumed throughout :

  \begin{enumerate}[i) ]
  \item $\Xset$ is a compact metric space,
  \item $\xi$ has continuous sample paths,
  \item $\tau$ is continuous.
  \end{enumerate}

Assume that $\Xset$ is endowed with a countably generated sigma-algebra $\mathcal{A}$ and let $\mu$ be a finite
measure on $(\Xset,\mathcal{A})$.
For $(u,\thres) \in \Xset \times \Rset$, Let $p_n(u,\thres) = \Esp_n \left( \one_{\xi(u) \geq \thres} (u) \right) %
= \Prob_{n} \left( \xi(u) \geq v \right)$.
In the considered integrated threshold weighted CRPS settings, our measure of residual uncertainty is
\begin{equation}
  \label{eq:uncertainty_bernoulli}
  H_n =  \int_{\Xset}  \int_{\Rset} p_{n}(u,\thres) \left( 1 - p_{n}(u,\thres) \right) \dmu(u) \dgamma(\thres),
\end{equation}
which corresponds to the uncertainty functional
\begin{equation}
  \label{equ:ibv:Hcal}
  \Hcal(\nu) = \int_{\Xset \times \Rset} p_\nu \left(1 - p_\nu\right) \dmugamma,
  \qquad \nu \in \MGCX,
\end{equation}
where $p_\nu(u,\thres) = \int_\Sset \one_{f(u) \geq \thres}\, \nu(\df)$.
The corresponding SUR sampling criterion is 
\begin{equation*}
  J_n(x) = \Esp_{n, x} \left(
    \int_{\Xset \times \Rset} p_{n+1} \left( 1 - p_{n+1} \right) \dmugamma
  \right).
\end{equation*}

Revisiting the construction presented in \citet{bect2019supermartingale} for the integrated Bernoulli variance, one may note that the functional~\eqref{equ:ibv:Hcal} can be seen as the uncertainty functional induced by the loss function
\begin{equation}
  \label{equ:ibv:loss}
  \begin{aligned}
    L: \Sset \times \Dset & \;\to\; \Rplus,\\
    (f, d) & \;\mapsto\; \lVert \one_{{\xi(u) \geq \thres}} -d \rVert_{L^2(\mu \times \gamma)}^{2},
  \end{aligned}
\end{equation}
where $\Dset \subset L^2(\mu \times \gamma)$ is
the set of $(\mu \times \gamma)$-square integrable measurable functions on~$\Xset \times \Rset$ taking values
in~$\left[ 0, 1 \right]$. %
Indeed, for all~$\nu \in \MGCX$ and $\xi \sim \nu$,
\begin{eqnarray*}
  {\overline L}_\nu(d) %
  = \Esp \left( L(\xi,d) \right) %
  = \lVert p_\nu - d \rVert_{ L^2(\Xset)}^{2} + \int p_\nu (1 - p_\nu)\, \dmu
\end{eqnarray*}
is minimal for $d = p_\nu$, and therefore
$\Hcal(\nu) = \inf_{d \in \Dset} {\overline L}_\nu (d)$.

The following theorem establishes the convergence of SUR (or
quasi-SUR) designs associated to this novel uncertainty functional, by using the
theory developed in \citet{bect2019supermartingale}.  
\setcounter{theorem}{3}
\begin{theorem} \label{thm:ItwCRIPS}
  The loss function~\eqref{equ:ibv:loss} is regular (in the sense of Definition 3.18 of \cite{bect2019supermartingale}. As a consequence, $\Hcal(\Prob_n^\xi) \xrightarrow[]{\as}  0$
  for any quasi-SUR design associated with~$\Hcal$.
\end{theorem}

\newcommand \ProofItem[1] {%
  \medskip \noindent \textbf{#1}}

\begin{proof}
  The proof consists in checking the six points of the definition of regular losses, as follows. Overall, the steps are closely following those made in Bect et al. for the integrated Bernoulli variance. e) and f) require a bit of care, with somehow surprisingly general settings in the case of point f).     

  \ProofItem{a)} %
  $\Dset$ is separable

  The Lebesgue space $L^2(\mu\times \gamma)$ is a separable metric space since the sigma-algebras equipping $\Xset$ and $\Rset$ are countably generated. 
  Hence $\Dset$ is also separable.

  \ProofItem{b)} %
  for all~$d \in \Dset$, $L(\cdot, d)$ is $\Scal$-measurable: 
  $f \mapsto \int_\Xset \int_\Rset \left( \one_{f(u)\ge v} - d(u) \right)^2\,
  \dmu(u)\dgamma(v)$ is $\Scal$-measurable by Fubini's theorem since the
  integrand is $\Scal \otimes \mathcal{A} \mathcal{B}(\Rset)$-jointly measurable
  in~$(f, u,v)$.

  \ProofItem{c)} %
  for all $\nu \in \MGCX$, ${\overline L}_\nu$ takes finite
  values and is continuous on~$\Dset$: ${\overline L}_\nu$ is finite since the loss is
  upper-bounded by~$\mu(\Xset)$, and its continuity follows from the continuity of the norm.

  \ProofItem{d)} %
  $\Hcal = \Hcal_0 + \Hcal_1$, where
  $\Hcal_0(\nu) = \int_\Sset L_0\, \dnu$ for some
  $L_0 \in \cap_{\nu \in \MGCX} \mathcal{L}^1 \left( \Sset, \Scal, \nu
  \right)$, and $\Hcal_1$ is \UiFgcd: holds with $L_0 = 0$ and $\Hcal_1 = \Hcal$ as $\Hcal$ is \UiFgcd since $L$ is upper-bounded.

  \ProofItem{e)} %
  $\Hcal_1$ is \AscFgcd. The proof closely follows \cite{bect2019supermartingale} up to moving to product spaces and measures, resulting in adapted random sets $A(\omega)$ and $B(\omega)$. 

  Let $\xi \sim \GP (m, k)$ and let $(\nug_n)$ be a sequence of random
  measures $\nug_n \in \Pfrak(\xi)$ such that a.s.
  $\nug_n \to \nug_{\infty} \in \Pfrak(\xi)$.
  For $n \in \Nset \cup \{ \infty \}$, let $m_n$ and $k_n$ be the (random) mean
  and covariance functions of~$\nug_n$.
  For $u \in \Xset$ and $n \in \Nset \cup \{ \infty \}$, let also
  $\sigma^2(u) = k(u,u)$, $\sigma^2_n(u) = k_n(u,u)$, and
  \begin{equation*}
    g_n(u,v) = g \left( \bar{\Phi} \left( \frac{v - m_n(u)}{\sigma_n(u)}
      \right) \right),
  \end{equation*}
  where $g(p) = p (1 - p)$ and $\bar{\Phi}(v) = P( Z \geq v )$ where
  $Z$ is a standard Gaussian variable, with the convention that
  $\bar{\Phi}(0/0) = 1$.
  We will prove below that, for all $n \in \Nset \cup \{ + \infty \}$,
  \begin{equation}
    \label{eq:int_A}
    \Hcal( \nug_n ) = \int_{\Xset} g_n(u,v)\, \dmu(u) \dgamma(v) \eqas
    \int_A g_n(u,v)\, \dmu(u) \dgamma(v)),
  \end{equation}
  where $A$ denotes the random subset of~$\Xset \times \Rset$ defined by
  \begin{equation*}
  A(\omega)=\{ (u,v) : \sigma(u)>0, \sigma_{\infty}(\omega,u) =
  0 , m_{\infty}(\omega,u) \neq v \}\cup \{ (u,v): \sigma(u) >0,
  \sigma_{\infty}(\omega,u) > 0 \}.
  \end{equation*}
  %
  %
Since $\nug_n \to \nug_{\infty}$ almost surely, it
  holds for almost all $\omega \in \Omega$ that
  $m_n(\omega, \cdot) \to m_{\infty}(\omega, \cdot)$ and
  $\sigma_n(\omega, \cdot) \to \sigma_{\infty}(\omega, \cdot)$
  uniformly on~$\Xset$.
  Furthermore, for each $(u,v) \in A(\omega)$, either
  $\sigma_{\infty}(\omega, u) >0$ or
  $\sigma_{\infty}(\omega, u) =0, m_{\infty}(\omega, u) \neq v$.
  In both cases, we have that: \\
  $g\left( \bar{\Phi}([m_n(\omega, u) - v]/ \sigma_n(\omega, u) )
  \right) \to g\left( \bar{\Phi}([m_{\infty}(\omega,
    u)-v]/\sigma_{\infty}(\omega, u)) \right)$.
  So, for almost all $\omega \in \Omega$ we obtain by dominated
  convergence theorem on $A(\omega)$ that
  \begin{equation*}
    \Hcal( \nug_n )= \int_A g_n(u,v)\, \dmu(u) \dgamma(v)
    \;\xrightarrow[n \to \infty]{\text{a.s.}}\;
    \Hcal( \nug_{\infty} )= \int_A g_{\infty}(u,v)\, \dmu(u) \dgamma(v).
  \end{equation*}


  To prove~\eqref{eq:int_A}, let us observe first that for any $u$ such that $\sigma(u) = 0$,
  $\sigma_n(u) \eqas 0$ for all $n \in \Nset \cup \{ \infty \}$ since $\nug_n \in \Pfrak(\xi)$.
  Hence, $g_n(u) \eqas 0$ when $\sigma(u) = 0$.
  %
  %
  Thus, setting
  $B(\omega) = \{ (u,v);\, \sigma(u) > 0,\,
  \sigma_{\infty}(\omega, u) = 0,\, m_{\infty}( \omega, u) = v \}$, we
  have for all $\omega \in \Omega$
  \begin{equation*}
    \Xset = \{ u \in \Xset, \sigma(u) = 0 \} \cup A( \omega ) \cup B( \omega ).
  \end{equation*}
  %
  %
  %
  In order to prove~\eqref{eq:int_A}, it is sufficient
  to show that
    $\int_B g_n(u,v)\, \dmu(u) \dgamma(v) \eqas 0$,
  %
  which we now establish by proving that $\mu(B) = 0$ almost surely.  First, %
  since $(\omega, u) \mapsto m_\infty(\omega,u)$ and
  $(\omega, u) \mapsto \sigma_\infty(\omega,u)$ are jointly measurable,
it follows from the Fubini-Tonelli theorem that
  \begin{equation} \label{eq:Domega}
    \Esp( (\mu\times \gamma)(B) ) =   \int_{\Xset \times \Rset }
    \one_{\sigma(u) > 0}\, \Esp( \one_{\sigma_{\infty}(u) = 0}
    \one_{m_{\infty}(u) = v} ) \dmu(u) \dgamma(v).
  \end{equation}
 
 \noindent     
  Directly using \cite{bect2019supermartingale}'s intermediate result that, for any $u \in \Xset$,
  the implication
  \begin{equation*}
    \sigma_{\infty}(u) = 0 \Longrightarrow
    \xi(u) = m_{\infty}(u) 
  \end{equation*}
  holds almost surely, we have for any $(u,v)$
  \begin{align*}
   0 \leq \one_{\sigma(u) > 0}\, \one_{\sigma_{\infty}(u) = 0}
    \one_{m_{\infty}(u) = v}
     =
      \one_{\sigma(u) > 0} \one_{\sigma_{\infty}(u) = 0} \one_{ \xi(u) = v} 
     \leq
      \one_{\sigma(u) > 0}  \one_{ \xi(u) = v}
    = 0
  \end{align*}
  almost surely, since $\xi(u) \sim \mathcal{N} ( 0 ,
  \sigma(u)^2)$.  
  Hence, the integrand in the right-hand side of~\eqref{eq:Domega} is zero,
  which implies that $\Esp\left( (\mu \times \gamma) (B) \right) = 0$, and therefore
    $(\mu \times \gamma) (B) \eqas 0$ since $(\mu \times \gamma)(B)$ is a non-negative random variable.
  Thus~\eqref{eq:int_A} holds and the proof of {\bf e)} is complete.

  \ProofItem{f)} 
  $\Zset_\Hcal = \Zset_\Gcal$

  Let $\nu \in \Zset_{\Gcal}$ and let $\xi \sim \nu$. Let
  $m,k,\sigma^2$ be defined as above.
  Let $U \sim \mathcal{N}(0,1)$ be independent of~$\xi$.
  Since $\Gcal(\nu) = 0$, we have from the law of
  total variance
  \begin{eqnarray*}
    \int_{\Xset \times \Rset } \var \left( \Esp
    \left( \one_{\xi(u) \geq T} | Z_x \right) \right) \dmu(u) \dgamma(v) = 0
  \end{eqnarray*}
  for all $x \in \Xset$, where $Z_x = \xi(x) + \tau(x) U$.
  Hence, for all $x \in \Xset$, we have
  \begin{equation*}
    \var \left(
      \bar{\Phi} \left(
        \frac{%
          v -  m(u) - \frac{k (x, u)\, (Z_x - m(x))}{\sigma^2(x) + \tau^2(x)}%
        }{%
          \sqrt{   \sigma^2(u) - \frac{k(x,u)^2}{\sigma^2(x)+\tau^2(x)}}
        }
      \right)
    \right) = 0,
  \end{equation*}
  for $(\mu,\gamma)$-almost every $(u,v)$.
  For such an $(u,v)$, the latter implies that for all $x \in \Xset$ $k (x, u) = 0$, where form follows   
  %
  %
  that $\sigma^2(x) = 0$ for all $x \in \Xset$, and therefore $\Hcal(\nu) = 0$.
\end{proof}

In the next proposition, we refine Theorem~\ref{thm:ItwCRIPS}. 

\begin{proposition} \label{prop:consist:p:un:moins:p}
  For any quasi-SUR design associated with $\Hcal$, as $n \to \infty$,
  almost surely and in~$L^1$,
  \begin{equation*}
    \int_{\Xset \times \Rset} \left( \one_{\xi(u) \geq v} - p_n(u,v) \right)^2 \dmu(u) \dgamma(v) \to 0
  \end{equation*}
  and
  \begin{equation*}
    \int_{\Xset \times \Rset} \left( \one_{\xi(u) \geq v}
      - \one_{p_n(u,v) \geq 1/2} \right)^2 \dmu(u) \dgamma(v) \to 0.
  \end{equation*}
\end{proposition}

\begin{proof}
  From steps \textbf{e)} and \textbf{f)} in the proof of
  Theorem~\ref{thm:ItwCRIPS}, it follows that
  \begin{equation*}
    \int_{\Xset} \left( \one_{\xi(u) \geq T} - p_n(u) \right)^2 \dmu(u) \dgamma(v) \eqas
    \int_A \left( \one_{\xi(u) \geq T} - p_n(u) \right)^2 \dmu(u) \dgamma(v).
  \end{equation*}
  Also, for almost all~$\omega \in \Omega$ and all
  $(u,v) \in A( \omega )$,
  $p_{n}(\omega, u,v) \to \one_{\xi(\omega, u) \geq v} $ as
  $n \to \infty$ since $\sigma_{\infty} \equiv 0$ a.s.\ from the proof
  of \textbf{f)} in Theorem~\ref{thm:ItwCRIPS} and the conclusion
  of this theorem.
  Hence the first part of the proposition follows by applying the
  dominated convergence theorem twice.
  The proof of the second part of the proposition is identical.
\end{proof}

\bibliographystyle{apalike} 
\bibliography{ref}